\documentclass[11pt]{article}
\usepackage[utf8]{inputenc}
\usepackage[margin=1.00in]{geometry}

\usepackage{graphicx}
\usepackage{amsmath}
\usepackage{amssymb}
\usepackage[linesnumbered,ruled]{algorithm2e}
\usepackage{color}
\usepackage{url}
\usepackage{cite}
\usepackage{siunitx}
\usepackage{hyperref}
\usepackage{float}

\usepackage{amsthm}
\usepackage{amsfonts}
\usepackage{subfigure}
\usepackage{mathtools}
\usepackage[normalem]{ulem}

\DeclarePairedDelimiter\floor{\lfloor}{\rfloor}

\newcommand{\Excess}{\textsc{\texttt{Excess}}}
\newcommand{\norm}[1]{\left\lVert#1\right\rVert}

\newtheorem{mydef}{Definition}
\newtheorem{theorem}{Theorem}
\newtheorem{lemma}{Lemma}

\newtheorem{corollary}{Corollary}
\newtheorem{prob}{Problem}

\newcommand{\eg}{\textit{e.g.}}
\newcommand{\ie}{\textit{i.e.}}
\newcommand{\etal}{\textit{et al}.}

\title{Online Exploration of an Unknown Region of Interest with a Team of Aerial Robots}
\date{}

\author{
  Yoonchang Sung, Deeksha Dixit, and Pratap Tokekar\footnote{Y. Sung is with the Computer Science Department, The University of Texas at Austin, Austin, TX, 78712 USA. D. Dixit and P. Tokekar are with the Department of Computer Science, University of Maryland, College Park, MD 20742, USA.}
}

\begin{document}



\maketitle

\begin{abstract}
In this paper, we study the problem of exploring an unknown Region Of Interest (ROI) with a team of aerial robots. The size and shape of the ROI are unknown to the robots. The objective is to find a tour for each robot such that each point in the ROI must be visible from the field of view of some robot along its tour.

We propose a recursive depth-first search-based algorithm that yields a constant competitive ratio for the exploration problem. Our analysis also extends to the case where the ROI is translating, \eg, in the case of marine plumes under constant wind conditions. In the simpler version of the problem where the ROI is modeled as a 2D grid, the competitive ratio is \sloppy $\frac{2(S_r+S_p)(R+\floor{\log{R}})}{(S_r-S_p)(1+\floor{\log{R}})}$ where $R$ is the number of robots, and $S_r$ and $S_p$ are the robot speed and the ROI speed, respectively. We also consider a more realistic scenario where the ROI shape is not restricted to grid cells but an arbitrary shape. We show our algorithm has $\frac{2(S_r+S_p)(18R+\floor{\log{R}})}{(S_r-S_p)(1+\floor{\log{R}})}$ competitive ratio under some conditions. We empirically verify our algorithm using simulations as well as a proof-of-concept experiment mapping a 2D ROI using an aerial robot with a downward-facing camera.

\

\textbf{Keywords:} {multi-robot systems, online exploration, aerial robots, competitive algorithm}
\end{abstract}

%

\section{Introduction}
We investigate the problem of exploring and mapping an unknown 2D Region Of Interest (ROI) using a team of aerial robots. Our overall vision is to develop coverage algorithms for enabling a team of robots to assist emergency responders in disaster scenarios or environmental scientists in data collection. Figure~\ref{fig:overview} shows one motivating scenario where autonomous Unmanned Aerial Vehicles (UAVs) can be used to map the region in a lake that is contaminated by a leaked pollutant (\eg, chemical spill in a lake). Instead of learning spatiotemporal phenomena or finding a hotspot in the contaminated region, we are interested in mapping the entire contaminated region, which could possibly move by the wind. The size and shape of this region are usually not known until observed by the UAVs. 
If the UAVs are flying at lower altitudes or if the contaminated region is large, the UAV will need to plan its motion to explore and map out the unknown ROI using its onboard sensors (\eg, downwards-facing camera in the case of visible contaminants such as chemical spills). 

Another example is assessing the damage to structures after a natural disaster~\cite{fernandez2015uav} using downwards-facing cameras mounted on UAVs. Such assessment is required to estimate the cost of recovery and reconstruction. As in the previous example, the exact shape of the region that has undergone damage may not be known a priori and must be mapped using the UAVs. There are many such examples of mapping of ROIs of unknown shape and size using robots equipped with appropriate sensors~\cite{stachniss2003exploring,masehian2017cooperative,song2020care}.

\begin{figure}[thpb]
\centering
\includegraphics[width=0.55\columnwidth]{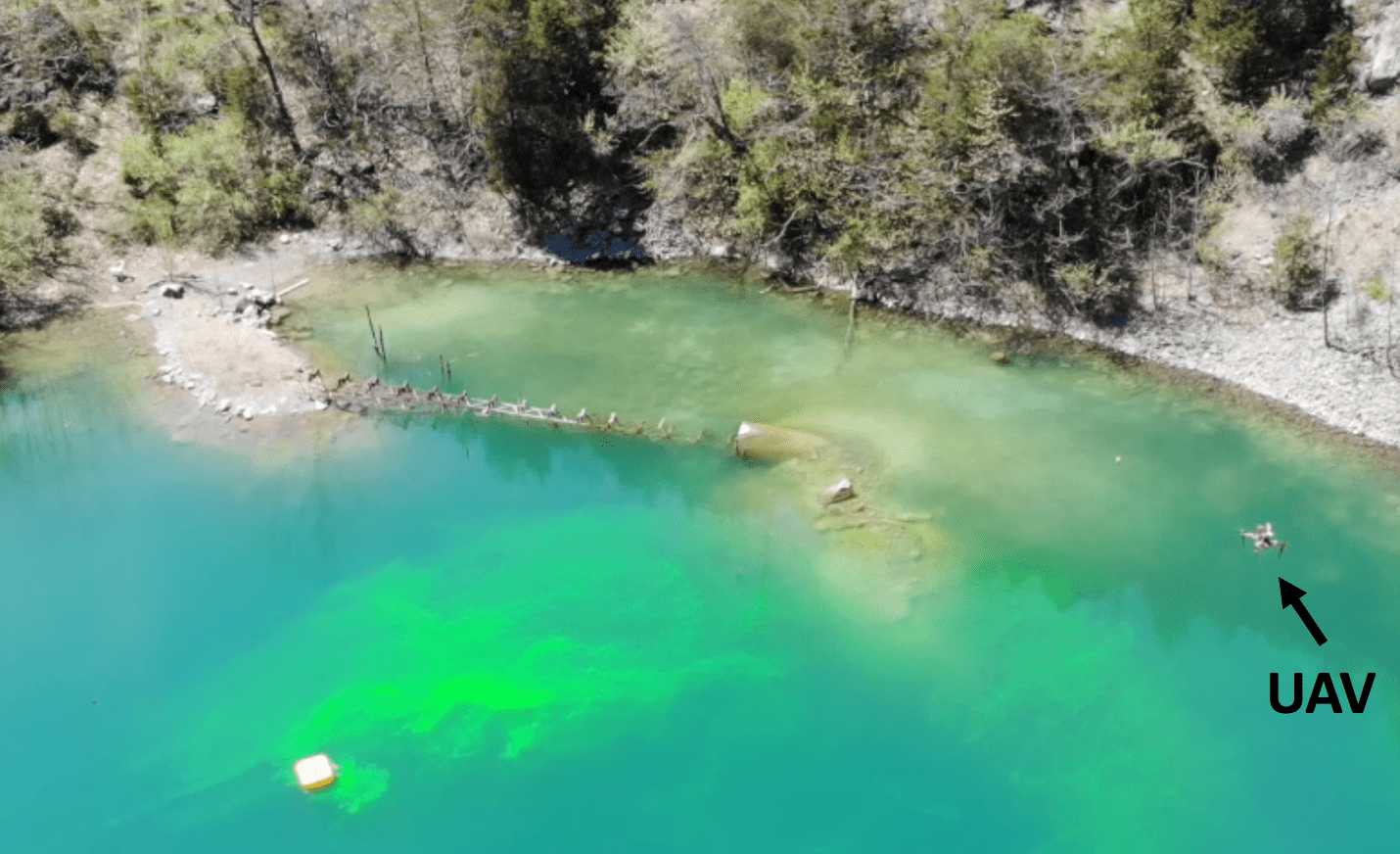}
\caption{A UAV conducting plume exploration in an abandoned quarry near Blacksburg, Virginia.}
\label{fig:overview}
\end{figure}

In such scenarios, the larger environment which contains the ROI is known but the exact shape of the ROI is unknown. 
We assume that the robot has a sensing pipeline that is capable of taking in the image and distinguishing whether the robot is above an ROI or not. The goal is to plan a trajectory for a team of aerial robots to collectively map the ROI in the least amount of time.
The problem of exploring an unknown 2D environment is a well-studied one in the robotics~\cite{brass2011multirobot,julia2012comparison,nuske2015autonomous,girdhar2014autonomous} and computational geometry~\cite{icking2005exploring,kolenderska2009improved,fraigniaud2006collective,higashikawa2014online} communities. However, the problem considered in this paper differs from these works in the following ways.


While the focus is on monitoring static scenes, our work extends to translating ROI, operating under constant wind conditions. In that case, we show that the performance of the algorithm is, not surprisingly, a function of the relative speeds of the robots and the ROI. Depending on the application, there may exist multiple ROIs in the environment. Furthermore, the robots may start in a region that is not part of the ROI and may have to find the ROI in the first place. In such cases, we can use a boustrophedon search pattern to find the ROI~\cite{tokekar2013tracking}. In this paper, we focus only on mapping a single connected ROI. The single ROI algorithm can be extended to multiple static ROIs.

We use the notion of \emph{competitive ratio}~\cite{borodin2005online} to analyze the performance of our algorithm. The competitive ratio for an online algorithm is defined as the largest (\ie, worst-case input) ratio of the time taken by the online algorithm to the time taken by an optimal offline algorithm. The offline algorithm is one which knows the shape of the 2D ROI a priori. We seek algorithms that have a low (preferably, constant) competitive ratio. Our main result is a constant competitive ratio for exploring a translating ROI for a fixed number of robots. The constant depends on the relative speeds of the ROI and the robots.

We require the robots to ensure that all points of the ROI are eventually covered by the sensor footprints of at least one of the robots along their paths.
The objective is to minimize the time required for all the robots to explore the ROI and return back to the starting position. Our algorithm builds on the one presented by Higashikawa~\etal~\cite{higashikawa2014online} for exploring an unknown binary tree. We show how to reduce the problem of exploring the ROI to that of exploring a binary tree. 
We first start with the simpler scenario where the ROI is modeled as a 2D grid and then extend it to translating ROI. We further generalize it to the case where the ROI boundary is any smooth (formally defined in Section~\ref{sec:prob}) 2D curve with a finite sensor footprint binary sensor. For both cases, we show that our algorithm yields a constant-competitive ratio.

We validate our algorithm through simulations that quantify the performance as a function of the size of the ROI, the number of robots, and the relative speeds of the ROI and the robots. We also conduct a proof-of-concept field experiment using a UAV with a downwards-facing camera to explore and map a stationary region of interest (runway). We discuss how to implement the algorithm in a practical setting and discuss challenges associated with noisy measurements.

In summary, the contributions of the paper are as follows:
\begin{itemize}
\item We propose a new exploration algorithm for a team of aerial robots that can completely map a region of interest of unknown size and shape. We consider a scenario where the environment is discretized into a grid of cells and the robots are equipped with a camera sensor 
\item We present a constant competitive ratio algorithm for this problem. Our analysis allows for arbitrarily shaped ROIs as well as possibly translating ones.
\item In addition to theoretical results, we evaluate our algorithm through simulations and proof-of-concept demonstrations.
\end{itemize}

A preliminary version of this paper was presented in Sung and Tokekar~\cite{sung2019competitive}. This version improves upon Sung and Tokekar~\cite{sung2019competitive} with a more expansive literature survey, a more detailed explanation on the proposed algorithm, and new simulation results and proof-of-concept experiments, including a description of how to implement the proposed algorithm using a robot with a downward-facing camera.

The rest of the paper is organized as follows. We begin by introducing the related work in Section~\ref{sec:related}. We describe the problem setup in Section~\ref{sec:prob}. Our proposed algorithm for a grid-based map is presented in Section~\ref{sec:grid}. We then extend this to arbitrarily shaped ROIs in Section~\ref{sec:arbitrary}. We present results from representative simulations in Section~\ref{sec:sim} and field experiments in Section~\ref{sec:experiment}, respectively, before concluding with a discussion of future work in Section~\ref{sec:conc}.

\section{Related Work}\label{sec:related}

Environmental monitoring has extensively been studied in robotics due to its practical applications. Some of highlighted tasks include precision agriculture~\cite{tokekar2016sensor,das2015devices}, wildlife habitat monitoring~\cite{tokekar2010robotic,tokekar2013tracking,plonski2017environment} and atmospheric plume tracking~\cite{ishida2012chemical,lochmatter2013plume,fahad2017robotic}. For survey results, see Dunbabin and Marques~\cite{dunbabin2012robots}. The area coverage and exploration are crucial for environmental monitoring as a given environment must be explored by robots in order to detect a target of interest. Galceran and Carreras~\cite{galceran2013survey} listed coverage path planning algorithms that can be used for different sensing and motion models. In case of ROI exploration, the aim is to explore and map an ROI by robots with limited sensing capability.

The objective of online exploration~\cite{julia2012comparison,nuske2015autonomous,girdhar2014autonomous} is to explore and map a region without having prior knowledge on the size and shape of the region.
We begin with several  problems in online exploration of the environment that do not necessarily focus on the limited Field Of View
(FOV) of the sensor, which is the main focus of this paper.
The goal of informative exploration is to maximize the information gathered along the planned trajectories.
Arora and Scherer~\cite{arora2017randomized} designed a near-optimal algorithm for single robot informative path planning where budget constraints exist.
Popovi\'c~\etal~\cite{popovic2020informative} proposed a 3D informative path planning algorithm for a UAV to monitor a terrain.
Corah and Michael~\cite{corah2019distributed} developed a near-optimal distributed algorithm for multi-robot exploration, which approximates the well-known sequential greedy assignment~\cite{singh2009efficient}.
The tracking problem has also been studied in regard to online exploration, such as
algorithms for tracking radio-tagged invasive fish using USVs and ground robots proposed by Plonski~\etal~\cite{plonski2017environment}. They proved competitive ratios for navigating an environment containing an unknown obstacle and energy-efficient solar exploration. 
Localization in online exploration settings is another related topic.
Hitz~\etal~\cite{hitz2014fully} focused on localizing interesting areas in an unknown environment using level set estimation to monitor hazardous cyanobacteria blooms in lakes. The objective in these works was not to completely map an unknown environment (which is the case in this paper) but to  maximize information gain, track and localize targets of interest.

In the online exploration literature, there are several path planning algorithms that have been proposed.
Sim and Little~\cite{sim2009autonomous} proposed a vision-based exploration and mapping solution for a single robot. 
Cesare~\etal~\cite{cesare2015multi} developed a multi-robot exploration algorithm for heterogeneous robots with limited communication and battery-life constraints. However, these works do not guarantee complete coverage.

Some online exploration algorithms guarantee completeness but do not exhibit competitiveness.
Bender~\etal~\cite{bender2002power} and Das~\etal~\cite{das2007map} addressed the problem of dealing with unlabelled (\ie, anonymous) vertices when exploring an unknown graph.
The former defined a pebble that can identify a vertex and found the number of pebbles required to map an unknown environment.
While the former considered the case of a single robot, the latter proposed a distributed version, allowing multiple robots to start from different vertices, and proved upper bounds on the time complexity of their algorithm. Their algorithms, however, do not yield a competitive ratio used as a performance measure in this paper.

When an ROI region 
can be represented by a grid polygon, there exists literature which explores a polygonal region not only completely but also competitively with respect to the optimal trajectory. This can be categorized into \emph{lawn mowing} and \emph{milling} where the former allows a robot to move outside the boundary of a polygon whereas the latter does not. Icking~\etal~\cite{icking2000exploring} proposed a strategy of generating a competitive tour 
for online milling which may contain holes. Icking~\etal~\cite{icking2005exploring} showed $\frac{4}{3}$--competitive algorithm for online milling without considering holes. The algorithms presented by Arkin~\etal~\cite{arkin2000approximation} have $(3+\epsilon)$--approximation for offline lawn mowing and $2.5$--approximation for offline milling. Kolenderska~\etal~\cite{kolenderska2009improved} developed an online milling algorithm of a grid polygon without holes that has a competitive ratio of $\frac{5}{4}$. However, aforementioned works did not take into account a multi-agent perspective. Although Arya~\etal~\cite{arya2001approximation} presented an approximation algorithm for milling where multiple robots can be deployed, their algorithm solves an offline problem. In this work, we pose an online milling version for multiple robots, taking into account their limited sensor FOV.


Previous works in computational geometry assumed specific properties of the region under exploration to ease the analysis. We restrict the ROI to satisfy a specific notion of fatness (defined in the next section).
Van der Stappen and Overmars~\cite{van1994motion} used the notion of $k$--fatness in motion planning with obstacles --- the smaller the value of $k$, the fatter the obstacle.
Efrat~\cite{efrat2005complexity} defined a $(\alpha,\beta)$--covered object if each angle of a triangle fully inside the object is at least $\alpha$ and each edge of this triangle is at least $\beta$ multiplied by the diameter of the object are satisfied.
Aloupis~\etal~\cite{aloupis2014triangulating} adopted the same notation of the fatness for the application of triangulating and guarding polygons.
Lee~\etal~\cite{lee2016structured} used a similar fatness for a triangulation of a planar region for multi-robot coverage.
These works exploited the fatness to prove the space complexity of their algorithms.
In this work, we also define the fatness for proving the competitive ratio for arbitrary ROI shape.

There are competitive algorithms designed for the single-robot case.
Gabriely and Rimon~\cite{gabriely2001spanning} proposed a spanning tree-based coverage algorithm and studied running time and space requirement with respect to the number of cells in the environment.
Klein~\etal~\cite{klein2015local} considered the problem of covering an unknown contamination that expands over time. They proved the upper bound on the number of steps required for the robot to completely cover the contamination.
Sharma~\etal~\cite{sharma2019optimal} proposed a constant-factor approximation algorithm for a square-shaped robot to explore an unknown polygonal environment.

When multiple robots are considered, most of works~\cite{fraigniaud2006collective,brass2011multirobot,higashikawa2014online,mahadev2017mapping,dynia2007robots,preshant2016geometric} have studied a tree-based exploration by employing a recursive Depth-First Search (DFS). 
In these works, the environment to be explored was assumed to be a tree.
Fraigniaud~\etal~\cite{fraigniaud2006collective} proposed a tree exploration algorithm using $R$ robots that is $\mathcal{O}(\frac{R}{\log{R}})$--competitive. In their work, each robot was allowed to observe the incident edges but not the adjacent vertices. 
Brass~\etal~\cite{brass2011multirobot} used the same sensing model and improved the competitive ratio of Fraigniaud~\etal~\cite{fraigniaud2006collective} to $2|E|/R+\mathcal{O}((R+r)^{R-1})$, where $|E|$ and $r$ denote the number of edges and radius of the graph, respectively.
Dynia~\etal~\cite{dynia2007robots} improved the lower bound proposed by Fraigniaud~\etal~\cite{fraigniaud2006collective} of $2-\frac{1}{R}$ to $\Omega(\frac{\log{R}}{\log{\log{R}}})$. 
As a dual problem, instead of finding competitive trajectories for given robots, Das~\etal~\cite{das2015collaborative} presented an algorithm for minimizing the number of robots given limited energy $E$ for each robot.
Megow~\etal~\cite{megow2012online} showed that the competitive ratio of a single-robot DFS is $2(2+\epsilon)(1+2/\epsilon)$, where $\epsilon$ is a fixed positive parameter, when applied to general graphs. 
Higashikawa~\etal~\cite{higashikawa2014online} presented a $\frac{R+\floor{\log{R}}}{1+\floor{\log{R}}}$--competitive algorithm for exploring a binary tree with $R$ robots.
Preshant~\etal~\cite{preshant2016geometric} showed that the competitive ratio remains largely the same, $\frac{2(\sqrt{2}R+\log{R})}{1+\log{R}}$, where the environment was an orthogonal polygon\footnote{An orthogonal polygon is one in which the edges are aligned with either the $X$ or $Y$ axes.} but was modeled as a tree. We build on this and generalize this to the case where the environment boundary is not necessarily orthogonal. In fact, it can be curved and may contain holes as well. Furthermore, we show how to adapt this algorithm to the case where the environment itself is translating.

To share information among multiple robots, global or local communication can be used. Das~\etal~\cite{das2007map} and Brass~\etal~\cite{brass2011multirobot} introduced bookkeeping devices to write local information on the vertex so that other robots can read this information when they visit the same vertex later. 
Lee~\etal~\cite{lee2016structured} proposed distributed online exploration algorithms assuming a fully connected network. In Higashikawa~\etal~\cite{higashikawa2014online}, robots can communicate with each other when they meet at the same vertex. We adopt the same model. 



\section{Problem Description}\label{sec:prob}
We consider the problem of mapping an ROI (Definition~\ref{def:plume_shape}) using a team with $R$ robots. The size and shape of the ROI are not known to the robots a priori. We use $P\in \mathbb{R}^2$ to denote the 2D ROI. Let $int(P)$ be the interior of $P$ and $\partial P$ be the boundary of $P$. 

We assume that each robot has a 
camera with a square footprint on the plane containing the ROI.
Without loss of generality, we assume that the side length of the square sensor footprint is $1$ in this work.

To avoid complications due to the trivial case of a small ROI, we assume that the ROI is at least as large as the sensor footprint of the robots. Specifically, we require the ROI to satisfy the following assumption.
\begin{mydef}~\label{def:plume_shape}\textbf{\emph{(Fat ROIs)}}
For any $p^\prime\in \partial P$, let $p\in int(P)$ be a point on the normal to $\partial P$ at $p^\prime$ such that $p$ is at a distance of $\frac{\sqrt{2}}{2}$ from $p^\prime$. Let $B(p)$ be an open ball of radius $\frac{\sqrt{2}}{2}$, \ie, $B(p)=\{q\mid \norm{p-q} _2 < \frac{\sqrt{2}}{2}\}$ where $q\in \mathbb{R}^2$. We say that the ROI $P$ is \emph{fat} if $B(p)$ lies completely inside $int(P)$ for all $p^\prime\in\partial P$.
\end{mydef}

\begin{figure}[thpb]
\centering
\includegraphics[width=0.60\columnwidth]{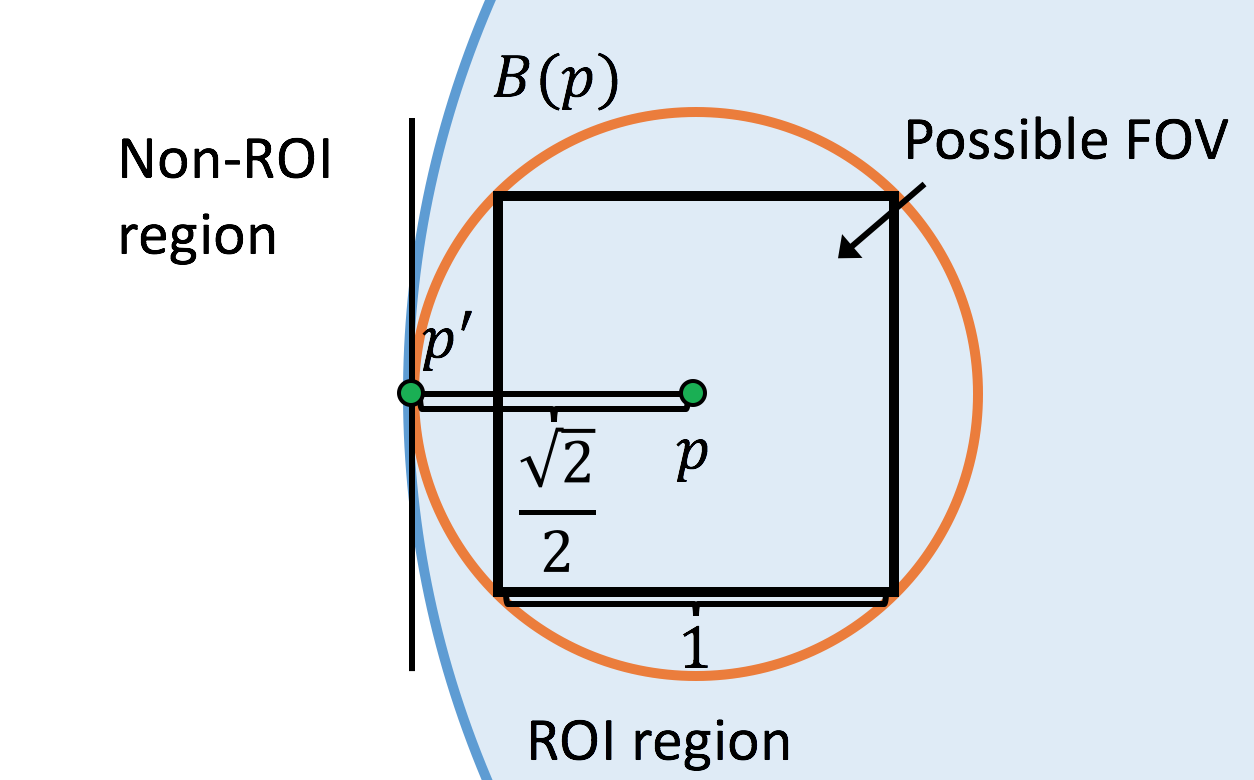}
\caption{We restrict our attention to ROIs that are \emph{fat} (Definition~\ref{def:plume_shape}).}
\label{fig:plume_shape}
\end{figure}

Figure~\ref{fig:plume_shape} shows an example of an ROI that is \emph{fat}. This definition disallows ROIs that have a \emph{width} less than that of the sensor footprint of the robot. Note, however, we still allow the ROI to contain one or more holes. 


While our focus is on mapping ROIs that are stationary, our analysis also extends to the case where the ROI is translating. For that case, we assume that the ROI translates with a fixed speed of $S_p$ in a fixed direction, both of which are known to the robots.\footnote{This is equivalent to the rigid-body translation of $P$.} For example, the velocity of a plume can be determined from the flow of the water which can be found from the environmental conditions such as wind and ocean current models~\cite{petrich2011board}.
We assume that all robots move at a speed of $S_r>S_p$. 

We focus on the mapping problem in this paper. Therefore, we assume that all robots start at the same location where they first observe the ROI. We seek tours for each robot that explore the ROI and return back to this starting location.

\begin{prob}~\label{prob:mrtp}\textbf{\emph{(Multi-Robot Exploration of Translating ROI)}}
Find a tour for all the robots that minimizes the exploration time such that every point in the ROI is visible from the sensor footprint of at least one robot's tour. All tours must return to the same starting position. The exploration time is given by the time when the last robot returns to the starting position.
\end{prob}


The proposed problem is an online exploration problem. The objective function is the exploration time which is the time of the longest tour.
In the next section, we present an algorithm that is based on recursive DFS which is competitive with respect to the optimal solution.

\section{ROI Exploration over a Grid Map}\label{sec:grid}


In this section, we present our main algorithm. We first solve a simpler version of Problem~\ref{prob:mrtp} where the ROI is approximated as a grid map. We then use this result to solve Problem~\ref{prob:mrtp} by relaxing the grid approximation afterwards. Our algorithm is based on the recursive DFS that models the ROI under exploration as a tree. We first show that our strategy is competitive for the grid map case and then analyze the effect of approximating an arbitrary ROI shape with a grid.


\subsection{Recursive DFS Algorithm for a Grid Map}\label{subsec:dfs}

In this section, we assume that the ROI is represented as a grid map~\cite{algfoor2015comprehensive}. The environment is modeled as a collection of cells, each of which is a square of unit side length. Each cell is connected to four of its neighbors. The ROI $P$ is just a collection of $C$ cells that form one connected set (if a cell $c\in P$ is part of the connected ROI $P$, then one of its four neighbors must also be a part of the ROI when $C>1$).

The problem of exploring the ROI is then simplified to that of exploring a grid map and identify the cells that belong in $P$. Since we assume that the sensor footprint is also a unit square, a robot may obtain an image by positioning itself at the center of a cell. By analyzing the pixels on the boundary of the image, the robot can then determine if any of the four neighboring cells are also part of the ROI or not. 

We model $P$ as a tree and propose a recursive DFS algorithm based on the tree exploration algorithm given by Higashikawa~\etal~\cite{higashikawa2014online}. Higashikawa~\etal~\cite{higashikawa2014online} developed a recursive DFS algorithm for exploring a binary tree. In our case, the grid graph to be explored is not necessarily a tree (it may contain cycles). Regardless, we show that modeling the underlying graph as a binary tree still leads to an algorithm with a constant competitive ratio.

The root of the tree is the cell corresponding to the starting position of the robots. Upon visiting a cell, the robots can identify if one or more of the four neighboring cells also contain the ROI. The neighboring cells that contain the ROI are added as children of the present cell in the tree unless those cells have been previously added to the tree. This condition prevents cycles.

The number of neighboring cells when a robot visits a new cell can be at most three. Therefore, the resulting tree may not be binary. However, by introducing a dummy edge of length $0$ and a dummy vertex, we can convert the tree into a binary tree without loss of generality.\footnote{This step is included in Line 15 of Algorithm~\ref{alg:dfs}.}

Each neighboring ROI cell determined by the sensing model becomes one of candidate cells that robots can choose from as the next vertex to visit. The goal becomes to visit all $C-1$ cells (that correspond to the ROI cells but excluding the starting cell) at least once by one of the robots. 

If $R=1$, then our algorithm becomes conventional recursive DFS for a single robot. However, in the multi-robot case, as the robots build the tree, we split the robots as equally as possible and assign them to explore the children vertices. 

We define three states for each vertex in the tree: \emph{unexplored} if the vertex is not visited by any robots; \emph{under exploration} if the vertex is visited by any robots but the leaf vertex connected from the vertex is not visited by any robots; and \emph{explored} if the vertex as well as the leaf vertex in the same branch are visited by any robots. When robots decide which vertex to move among neighboring cells of an ROI region, they do not consider \emph{explored} vertices but vertices that are either \emph{unexplored} or \emph{under exploration}. This is because having \emph{explored} vertex means that the offspring of it must have also been explored by any robots (see Figure~\ref{fig:tree}). 

\begin{figure}[htb]
\centering
\subfigure[Two robots exploring the grid map.]{\includegraphics[width=0.60\columnwidth]{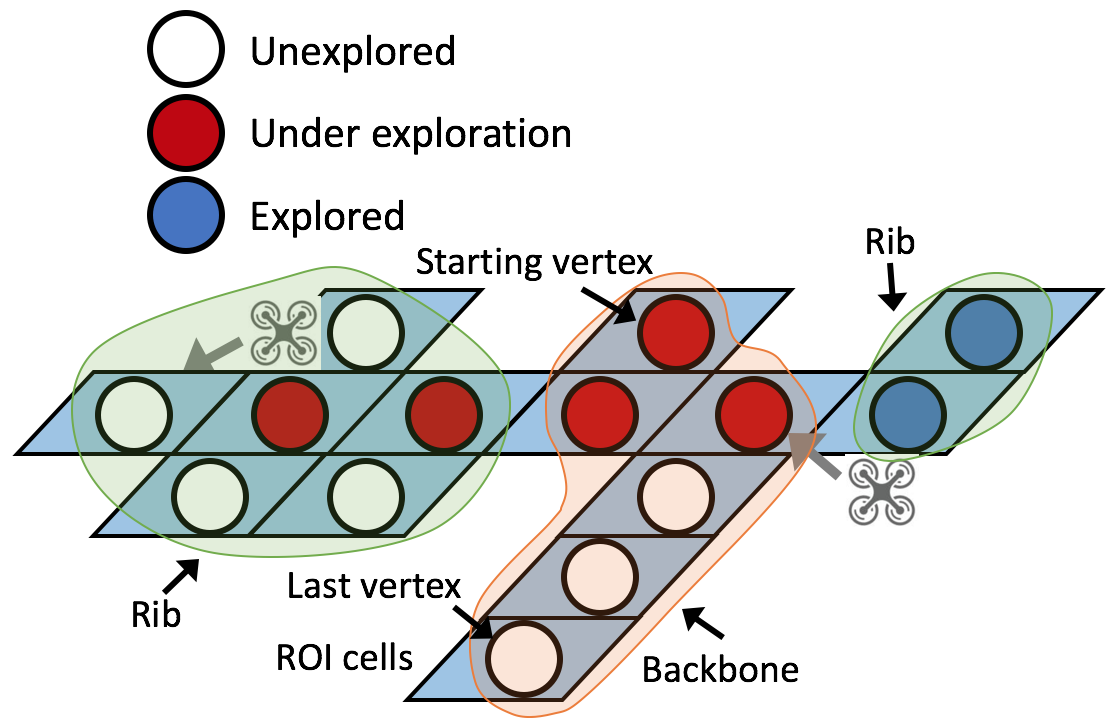}}
\subfigure[Tree generated from the recursive DFS.]{\includegraphics[width=0.35\columnwidth]{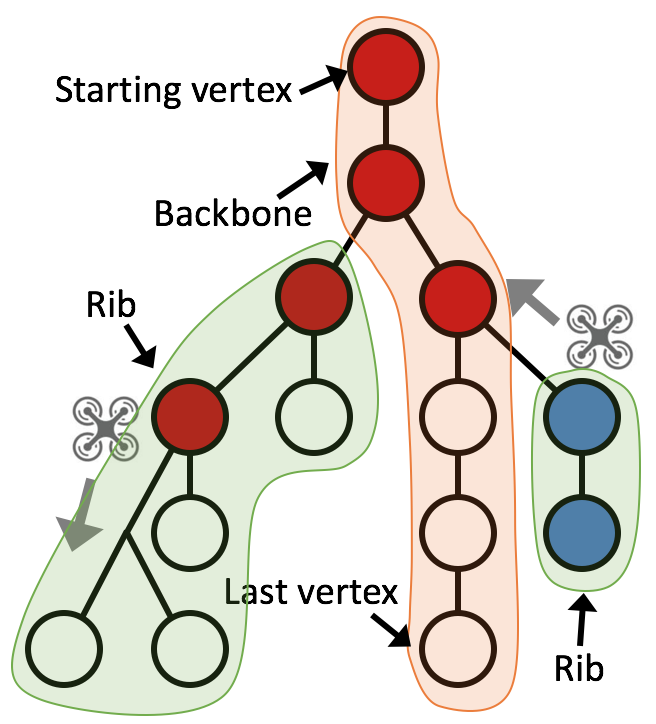}}
\caption{Description of tree components. The binary tree consists of a backbone and a finite number of ribs. Each vertex is marked as one of \emph{unexplored}, \emph{under exploration} or \emph{explored}.
}
\label{fig:tree}
\end{figure}

The details are given in Algorithm~\ref{alg:dfs}.\footnote{In the algorithm, we use $N(v_i)$ to denote the neighborhood of the $i$-th vertex such that $N(v_i)=\{v_j\in V|(v_j,v_i)\in E\}$.} All vertices are marked as \emph{unexplored} state in the beginning. Each robot runs Algorithm~\ref{alg:dfs} whenever it reaches a vertex. The algorithm can be implemented to a single robot independently with respect to other robots as long as they can share the state information of vertices. The robots need to share local information with each other only when they meet at the same vertex. Thus, the algorithm supports distributed communication while still preserving guarantees presented in Section~\ref{subsec:analysis}. The algorithm terminates when all robots return to the starting vertex and all vertices are marked as \emph{explored}.

\begin{algorithm}~\label{alg:dfs}
\SetAlgoLined
\SetKwInOut{Input}{Input}
\SetKwInOut{Output}{Output}

Observe $N(v)$ to determine whether neighboring cells are ROI cells or non-ROI cells.

\If {$|N(v)|$=0} {
Mark $v$ as \emph{explored}.

Move back to the parent vertex ($\rightarrow$next vertex) and directly jump to Line 24.
}

Communicate with robots to update the state of $N(v)$, \ie, \emph{unexplored}, \emph{under exploration}, and \emph{explored}.

$N(v)\leftarrow N(v)\backslash\{$\emph{explored} vertices$\}$.

\If{$v^\prime\in N(v)$ is under exploration} {
\If{moving to $v^\prime$ generates a cycle in the tree} {
$N(v)\leftarrow N(v)\backslash\{v^\prime\}$.
}
}

\uIf{$|N(v)|>1$} {
\If{$|N(v)|>2$} {
Add a dummy edge of length $0$ and a dummy vertex in order to keep the tree as a binary tree.
}

Split robots at $v$ into two children as equally as possible.

Move to one of two children ($\rightarrow$next vertex) and mark $v$ as \emph{under exploration}.
}
\uElseIf {$|N(v)|=1$} {
Move to the child ($\rightarrow$next vertex) and mark $v$ as \emph{under exploration}.
}
\ElseIf {$|N(v)|=0$} {
Move back to the parent vertex ($\rightarrow$next vertex).
}

$v\leftarrow $the next vertex.

\caption{Multi-Robot Recursive DFS}
\end{algorithm}

\subsection{Theoretical Analysis}\label{subsec:analysis}

In this section we analyze the proposed Algorithm~\ref{alg:dfs}. 
We start with the upper bound analysis. We then show the lower bound for optimal algorithm, followed by the competitive analysis for the case of grid approximation.

\paragraph{Upper Bound Analysis}


To analyze the cost of the proposed algorithm, we adapt the reward\footnote{We use the term \emph{reward function} to replace the term \emph{token} defined in Higashikawa~\etal~\cite{higashikawa2014online}.} collecting rule proposed by Higashikawa~\etal~\cite{higashikawa2014online} to the case of a translating ROI.
The reward measures how beneficial it is to visit unexplored vertices and is used to derive theoretical bounds with respect to the optimal reward that can be obtained by exploring a known graph.
Note that this rule is not required for implementing the algorithm, but only for analyzing the competitive ratio.  

Higashikawa~\etal~\cite{higashikawa2014online} define the concept of a backbone and a rib in a tree (shown in Figure~\ref{fig:tree}). The backbone is a path that starts from the root vertex and ends at one of the leaf vertices. The rib is a subtree generated by discarding the backbone and edges incident with the backbone from the original tree.

Let $l(e)$ be the length of an edge $e$. The length of an edge $e$ is $0$ if $e$ is dummy edge or $1$ otherwise. $L=\sum_{e\in E}l(e)$ be the sum of the total length of all edges in the tree. 
Note that $L=C-1$ where the ROI consisting of $C$ cells is represented by the tree structure.

Higashikawa~\etal~\cite{higashikawa2014online} define two reward functions, each with a total reward of $l(e)$, on every rib edge $e$ and $1+\floor{\log{R}}$ reward functions, again each with a total reward of $l(e)$, on every backbone edge $e$. 
The rewards are collected continuously by the robots following the rules described next: 
(1) Only one robot in a group traversing a rib edge in the forward direction for the first time collects a reward. 
(2) Only one robot in a group traversing a rib edge in the backward direction for the first time collects a reward.  
(3) Each of the 
$1+\floor{\log{R}}$ robots traversing a backbone edge in the forward direction for the first time collects a reward.
(4) Only one robot in a group traversing a backbone edge after the first group collects a reward.


Let $t_{last}$ be the time when the last robot reaches a leaf vertex in the tree. Higashikawa~\etal~\cite{higashikawa2014online} show that the total sum of rewards collected by all the robots is at least $(1+\floor{\log{R}})t_{last}$. This assumes that the robots move at unit speed and the tree is static. In our case, the tree (actually, ROI) is moving with a speed of $S_p$ and the robots are moving with a speed of $S_r$. The reward collection rule does not change in this case. What changes is the total sum of rewards collected in $t_{last}$ time. 
In our case, the total sum of rewards collected by all the robots will be at least $(S_r-S_p)(1+\floor{\log{R}})t_{last}$. The term $(S_r-S_p)$ comes from the lower bound on the relative speeds of the robots and the ROI.

Higashikawa~\etal~\cite{higashikawa2014online} also show that the total possible reward that the robots can collect is at most $2(L-d_{max})+(1+\floor{\log{R}})d_{max}$. Here $d_{max}$ denotes the distance of the farthest vertex in the tree from the root. Therefore, we have:
\begin{equation}~\label{eqn:higashikawa}
\begin{split}
&(S_r-S_p)(1+\floor{\log{R}})t_{last} \\
&\le 2(L-d_{max})+(1+\floor{\log{R}})d_{max}.
\end{split}
\end{equation}

We denote the time taken by the proposed algorithm by $\texttt{ALG}$. We are now ready to state the upper bound on $\texttt{ALG}$.
\begin{lemma}[Upper Bound for Multi-Robot Recursive DFS]~\label{lemma:upper_alg}
\begin{equation}~\label{eqn:alg}
\begin{split}
\texttt{ALG}\le\frac{2(C+d_{max}\floor{\log{R}})}{(S_r-S_p)(1+\floor{\log{R}})}.
\end{split}
\end{equation}
\end{lemma}
\begin{proof}
$\texttt{ALG}$ can be upper bounded as follows:
\begin{equation}~\label{eqn:upper_alg_proof_1}
\begin{split}
\texttt{ALG}&\le t_{last}+\frac{d_{max}}{S_r-S_p}, 
\end{split}
\end{equation}
where $\frac{d_{max}}{S_r-S_p}$ is the time taken to traverse the longest length of the backbone when the robot and the ROI move away from each other. By using Equation (\ref{eqn:higashikawa}), we have: 
\begin{align}
t_{last}+\frac{d_{max}}{S_r-S_p}&\le\frac{2L+(\floor{\log{R}}-1)d_{max}}{(S_r-S_p)(1+\floor{\log{R}})}+\frac{d_{max}}{S_r-S_p}, \label{eqn:upper_alg_proof_2}\\
&=\frac{2(C-1+d_{max}\floor{\log{R}})}{(S_r-S_p)(1+\floor{\log{R}})}, \label{eqn:upper_alg_proof_3}\\
&\le\frac{2(C+d_{max}\floor{\log{R}})}{(S_r-S_p)(1+\floor{\log{R}})} \label{eqn:upper_alg_proof_4}.
\end{align}

Equation (\ref{eqn:upper_alg_proof_3}) is obtained by the fact that $L=C-1$. Removing a negative term from Equation (\ref{eqn:upper_alg_proof_3}) completes the proof as Equation (\ref{eqn:upper_alg_proof_4}).

\end{proof}

\begin{corollary}[Special Cases]~\label{corollary:special}
Upper bounds for the following special cases can be derived from Lemma~\ref{lemma:upper_alg}, such as Multi-Robot Static ROI (MRSR), Single Robot Translating ROI (SRTR), and Single Robot Static ROI (SRSR).

\begin{table}[h]
\centering
\begin{center}
\begin{tabular}{c|c} 
 \hline
 \hline \\[-1em]
 \textbf{MRSR} & $\ \texttt{ALG}\le\frac{2(C+d_{max}\floor{\log{R}})}{S_r(1+\floor{\log{R}})}\ $ \\
 \hline \\[-1em] 
 \textbf{SRTR} & $\ \ \ \texttt{ALG}\le\frac{2C}{S_r-S_p}\ \ \ $ \\
 \hline \\[-1em] 
 \textbf{SRSR}
 & $\ \ \ \texttt{ALG}\le\frac{2C}{S_r}\ $ \\
 \hline
 \hline
\end{tabular}
\end{center}
\caption{Upper bounds of special cases.}
\label{table:1}
\end{table}

Note that the upper bound for MRSR becomes the result from Higashikawa~\etal~\cite{higashikawa2014online} if $S_r=1$. Also, the upper bound for SRSR is equivalent to Icking~\etal~\cite{icking2000exploring} if $S_r=1$.
\end{corollary}

\begin{proof}
The upper bound for MRSR can simply be obtained by plugging $S_p=0$ into Equation (\ref{eqn:alg}) of Lemma~\ref{lemma:upper_alg}.

The upper bound for SRTR can be derived from the upper bound of MRSR by having $R=0$. However, we can even tighten the bound by using the following observation: if the robot and the ROI move toward each other in one direction, they must move away from each other in order to return to the starting location, and vice versa. Therefore, $\texttt{ALG}$ can be upper bounded as:
\begin{equation}~\label{eqn:corollary_2_proof_1}
\texttt{ALG}\le\frac{C-1}{S_r+S_p}+\frac{C-1}{S_r-S_p}.
\end{equation}

Taking out negative terms from the above equation becomes:
\begin{equation}~\label{eqn:corollary_2_proof_2}
\begin{split}
\texttt{ALG}\le\frac{2S_rC}{(S_r+S_p)(S_r-S_p)},
\end{split}
\end{equation}
which is a tighter bound than $\frac{2C}{S_r-S_p}$. Note that the difference between these bounds is $\frac{S_r}{S_r+S_p}$ that satisfies $\frac{1}{2}<\frac{S_r}{S_r+S_p}\le 1$ because $S_r>S_p$.

The upper bound for SRSR can be derived by plugging either $R=1$ and $S_p=0$ into the upper bound for MRSR or $S_p=0$ into the upper bound for SRTR.
\end{proof}





\paragraph{Lower Bound Analysis}
We study the lower bound for the optimal algorithm in order to obtain a competitive ratio. Let ${\texttt{OPT}}^1_g$ be the time taken by the optimal algorithm to explore a grid map when using a single robot. The lower bound can be constructed as:
\begin{equation}~\label{eqn:single_optimal}
{\texttt{OPT}}^1_g\ge \frac{C-1}{S_r+S_p}.
\end{equation}

We use ${\texttt{OPT}}^R_g$ to represent the time taken by the optimal algorithm over any grid polygon of an ROI region using $R$ robots. Then, the following lemma gives the lower bound for ${\texttt{OPT}}^R_g$.

\begin{lemma}[Lower Bound for Optimal Algorithm]~\label{lemma:optimal}
\begin{equation}~\label{eqn:multi_optimal}
\begin{split}
{\texttt{OPT}}^R_g\ge \frac{C-1}{(S_r+S_p)R}.
\end{split}
\end{equation}
\end{lemma}
\begin{proof}
We claim the following inequalities.
\begin{equation}~\label{eqn:multi_optimal_proof_1}
\begin{split}
{\texttt{OPT}}^R\le {\texttt{OPT}}^1,
\end{split}
\end{equation}

This can be obtained from the fact that the more number of robots are deployed, the shorter time will be taken to explore the entire tree. 

Consider a tree consisting of $R$ branches. Then, we claim the following inequality: 
\begin{equation}~\label{eqn:multi_optimal_proof_2}
\begin{split}
{\texttt{OPT}}^1\le R{\texttt{OPT}}^R,
\end{split}
\end{equation}

Since $\texttt{OPT}^R$ is the time for a robot to explore the longest branch in the tree, $R{\texttt{OPT}}^R$ must be no less than ${\texttt{OPT}}^1$.

Combining these inequalities and Equation (\ref{eqn:single_optimal}), we prove Lemma~\ref{lemma:optimal}.
\end{proof}

It should be noted that the lower bound in Lemma~\ref{lemma:optimal} is loose, as shown in Figure~\ref{fig:simulation}, and thus improving the lower bound is of interest.

\begin{theorem}[Competitive Ratio over the Grid Polygon]~\label{theorem:competitive_grid}
The competitive ratio of Algorithm~\ref{alg:dfs} for a grid map is:
\begin{equation}~\label{eqn:competitive_grid}
\begin{split}
\texttt{ALG}\le&\frac{2(S_r+S_p)(R+\floor{\log{R}})}{(S_r-S_p)(1+\floor{\log{R}})}\texttt{OPT}^R_g\\
&+\frac{2}{(S_r-S_p)(1+\floor{\log{R}})}.
\end{split}
\end{equation}
\end{theorem}

\begin{proof}
Substituting Equation (\ref{eqn:multi_optimal}) into Equation (\ref{eqn:alg}) gives:
\begin{equation}~\label{eqn:proof_competitive_grid_1}
\begin{split}
\texttt{ALG}\le\frac{2((S_r+S_p)R\texttt{OPT}^R_g+1+d_{max}\floor{\log{R}})}{(S_r-S_p)(1+\floor{\log{R}})}.
\end{split}
\end{equation}

Since $\frac{d_{max}}{S_r+S_p}\le \texttt{OPT}^R_g$, it follows:
\begin{equation}~\label{eqn:proof_competitive_grid_2}
\begin{split}
\texttt{ALG}\le\frac{2(S_r+S_p)(R+\floor{\log{R}})\texttt{OPT}^R_g+2}{(S_r-S_p)(1+\floor{\log{R}})}.
\end{split}
\end{equation}

\end{proof}



\section{ROI Exploration over an Arbitrary ROI Shape}\label{sec:arbitrary}

The presented results so far are for a grid map approximation of the ROI. In this section, we will relate the bounds obtained for the grid map case to the case of arbitrarily shaped ROIs. Specifically, we will extend Lemma~\ref{lemma:upper_alg} to apply to an ROI region that may have an arbitrary shape. 

The algorithm for exploring the ROI remains the same. We will still construct a tree that represents a grid map of the ROI. The main difference here is that in the previous analysis, we assumed that the boundary of the ROI matched the boundary of a grid map exactly. This will no longer hold. Instead, we will explore a grid map that is an \emph{outer} approximation of the ROI (Figure~\ref{fig:in_cell}). 




We define $C^{\textit{ALG}}_{out}$ and $C^{\textit{ALG}}_{in}$ to denote the number of cells in the outer and inner grid approximation by our algorithm, respectively. The outer grid map completely contains the ROI whereas the inner grid map lies completely inside the ROI. Therefore, the term $C$ in the upper bound (Lemma~\ref{lemma:upper_alg}) will now be replaced by $C^{\textit{ALG}}_{out}$. However, the $C$ term in the lower bound (Lemma~\ref{lemma:optimal}) cannot be replaced by $C^{\textit{ALG}}_{in}$. This is because $C^{\textit{ALG}}_{in}$ is defined by the grid imposed by our algorithm. It may be possible to have another grid map (of the same unit side length) that is oriented and/or translated such that it contains fewer than $C^{\textit{ALG}}_{in}$ cells in the interior. We will first find the relationship between $C^{\textit{ALG}}_{out}$ and $C^{\textit{ALG}}_{in}$. Then, we will relate $C^{\textit{ALG}}_{in}$ to $C^{\textit{BEST}}_{in}$ which is the best grid that contains the fewest number of cells completely inside the ROI.

By a slight abuse of notation, we interchangeably use $C^{\textit{ALG}}_{out}$ and $C^{\textit{ALG}}_{in}$ to also denote the corresponding set of cells (along with denoting the number of cells in the set).

\begin{lemma}[Grid Approximation of Arbitrary ROI Shape]~\label{lemma:arbitrary}
The upper bound on $C^{\textit{ALG}}_{out}$ for a \emph{fat} polygon (from Definition~\ref{def:plume_shape}) is given by:
\begin{equation}~\label{eqn:arbitrary}
C^{\textit{ALG}}_{out}\le 3C^{\textit{ALG}}_{in}+6.
\end{equation}
\end{lemma}
\begin{proof}

\begin{figure}[thpb]
\centering
\includegraphics[width=0.70\columnwidth]{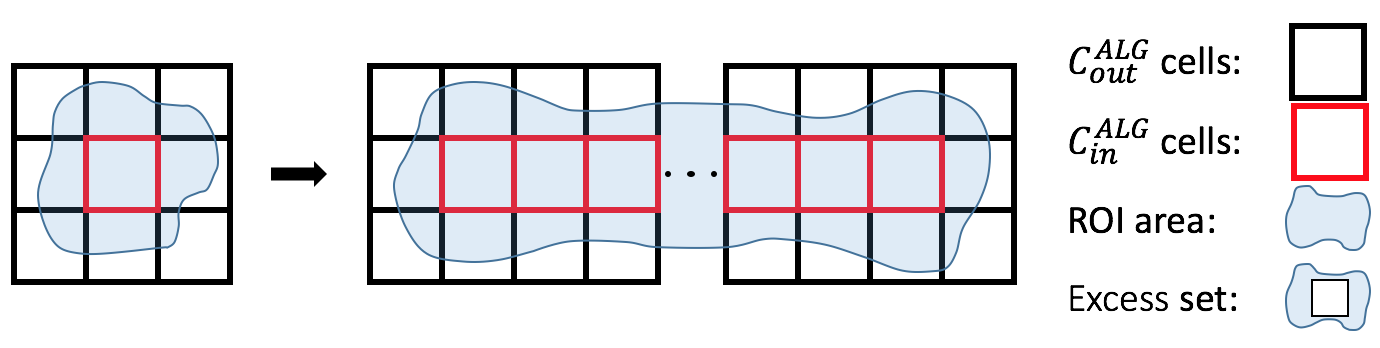}
\caption{Row formation of $C^{\textit{ALG}}_{in}$ cells as the number of cells changes from $1$ to a finite number.}
\label{fig:in_cell}
\end{figure}

To prove the lemma, we define an \Excess~set that contains all cells, $C^{\textit{ALG}}_{out}\backslash C^{\textit{ALG}}_{in}$. That is, \Excess\ set contains all cells in $C^{\textit{ALG}}_{out}$ but not in $C^{\textit{ALG}}_{in}$. Therefore, the size of the \Excess~set is equal to $C^{\textit{ALG}}_{out}-C^{\textit{ALG}}_{in}$. We prove the lemma in three steps.


\Excess~is maximum if and only if all the corners of all cells in $C^{\textit{ALG}}_{in}$ are also corners of all cells in the \Excess~set. 
In other words, if there exists at least one corner of any cell in $C^{\textit{ALG}}_{in}$ not belonging to a cell in the \Excess~set, \Excess~is not maximum. An example of such case is where $C^{\textit{ALG}}_{in}$ is $2\times 2$ set of four cells where the corner in the middle does not belong to any cells in the \Excess~set.

In fact, if the condition for minimum \Excess~is satisfied, then there cannot be any $2\times 2$ set of cells all belonging to $C^{\textit{ALG}}_{in}$. It can be easily shown that we can always convert any shapes of cells in $C^{\textit{ALG}}_{in}$ into a single row formation of cells (shown in the middle in Figure~\ref{fig:in_cell}) without decreasing the number of cells in the \Excess~set. 

If the single row formation is formed by a single cell, then the cell in $C^{\textit{ALG}}_{in}$ contributes two cells that are in the \Excess~set (one above and one below) in addition to three more cells on either end point. This is shown in the left in Figure~\ref{fig:in_cell}. Therefore, by generalizing this to the case of any finite number of cells, the size of \Excess~set becomes $2C^{\textit{ALG}}_{in} + 6$.



\end{proof}

\begin{figure}[thpb]
\centering
\includegraphics[width=0.55\columnwidth]{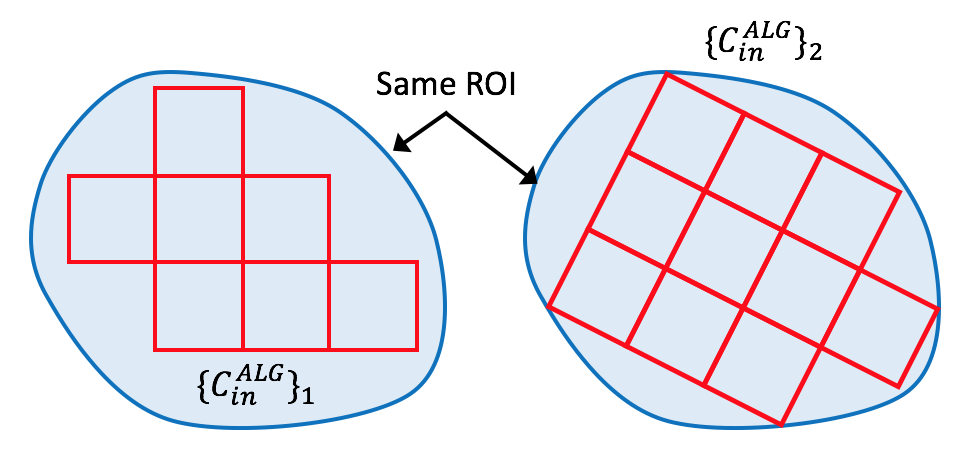}
\caption{
Example demonstrating how the grid approximation over arbitrary ROI affects the required number of cells to be visited to completely explore the ROI. Two grids with respect to different coordinate frames are applied to the same ROI. As $\{C^{\textit{ALG}}_{in}\}_1$ has fewer cells than $\{C^{\textit{ALG}}_{in}\}_2$, $\{C^{\textit{ALG}}_{in}\}_1$ constructs a smaller tree.}

\label{fig:grid_comparison}
\end{figure}

The grid corresponding to $C^{\textit{ALG}}_{out}$ and $C^{\textit{ALG}}_{in}$ is generated by the proposed algorithm. It is possible that there exists some other grid which has fewer than $C^{\textit{ALG}}_{in}$ cells completely contained within the ROI. 
Since the ROI has an arbitrary shape, a fewer number of cells covering the entire ROI leads to a fewer number of vertices in the tree, which would reduce the exploration time (example is included in Figure~\ref{fig:grid_comparison}).
It may not be possible to generate this ``best'' grid due to the nature of online exploration. Nevertheless, we analyze the relationship between $C^{\textit{ALG}}_{in}$ and $C^{\textit{BEST}}_{in}$. We define $C^{\textit{BEST}}_{in}$ to denote the fewest number of cells in the inner grid approximation that is completely contained in the ROI (and adding any other cell to $C^{\textit{BEST}}_{in}$ would not allow $C^{\textit{BEST}}_{in}$ to be completely inside the ROI). The relationship is given by:

\begin{lemma}[Best Possible Grid-Approximation]~\label{lemma:best}
\begin{equation}~\label{eqn:continuous}
C^{\textit{ALG}}_{in}\le 6 C^{\textit{BEST}}_{in}.
\end{equation}
\end{lemma}
\begin{proof}
To prove this relationship, it is sufficient to consider any grid approximation (generated by any algorithm) with respect to the best grid approximation. 

\begin{figure}[thpb]
\centering
\includegraphics[width=0.35\columnwidth]{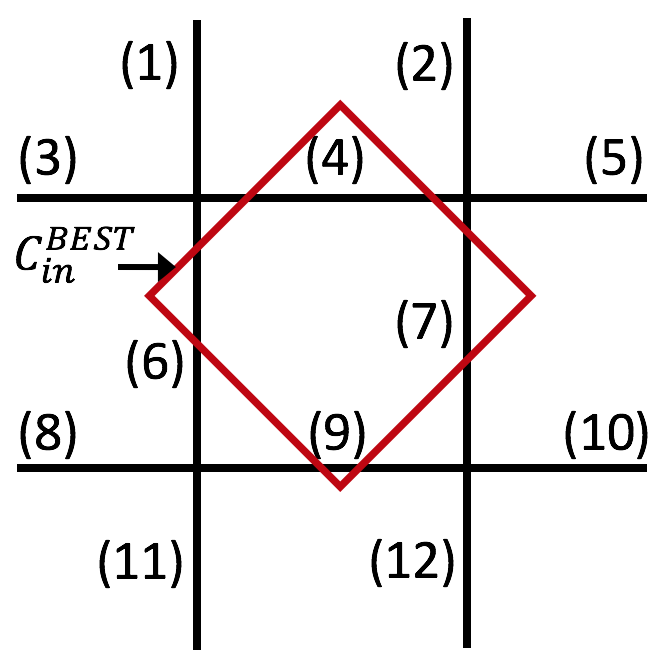}
\caption{A part of grid cell from any grid approximation. Unique number is assigned to a different side of grid cells.}
\label{fig:best_proof_1}
\end{figure}

Figure~\ref{fig:best_proof_1} shows a part of grid cells generated by any grid approximation.
Each number in the figure corresponds to a different side of grid cells. Let $C^{\textit{BEST}}_{in}$ be a single grid cell generated from the best grid approximation that overlaps with the central cell $(4,6,7,9)$ without loss of generality. Our observation is that the number of crossings is equal to the number of cells in $C^{\textit{ALG}}_{in}$ that $C^{\textit{BEST}}_{in}$ overlaps. 





We prove that 7 crossings are impossible.
In order to cross more than four edges, $C_{in}^{BEST}$ has to cross all of the (4, 6, 7, 9) edges. 
In addition, it must cross three of $(1, 2, 3, 5, 8, 10, 11, 12)$ edges. Let us consider the case when edge (1) is crossed. The other cases are symmetric. If edge (1) is crossed, then crossing (5, 8, 10, 11, 12) is impossible since these edges are more than a unit distance apart. Only (2) and (3) edges are only available edges to cross more. However, if $C^{\textit{BEST}}_{in}$ crosses both (2) and (3) edges, the side length of $C^{\textit{BEST}}_{in}$ becomes greater than $1$. Therefore, $C^{\textit{BEST}}_{in}$ cannot cross more than seven edges. Therefore, $C^{\textit{ALG}}_{in} \leq 6 C^{\textit{BEST}}_{in}$. 


\end{proof}

Finally we give our main result as follows:

\begin{theorem}[Competitive Ratio for Arbitrary ROI Shape]~\label{theorem:competitive_arbitrary}
Let ${\texttt{OPT}}^R$ be the time taken by the optimal algorithm over any arbitrary ROI shape using $R$ robots.
\begin{equation}~\label{eqn:main}
\begin{split}
\texttt{ALG}\le&\frac{2(S_r+S_p)(18R+\floor{\log{R}})}{(S_r-S_p)(1+\floor{\log{R}})}{\texttt{OPT}}^R\\
&+\frac{48}{(S_r-S_p)(1+\floor{\log{R}})}.
\end{split}
\end{equation}

\end{theorem}

\begin{proof}
Although ${\texttt{OPT}}^R$ is the cost for any arbitrary ROI shape, we can still lower bound this using $C^{\textit{BEST}}_{in}$ (similar to Lemma~\ref{lemma:optimal}) as:
\begin{equation}~\label{eqn:proof_competitive_arbitrary}
\begin{split}
{\texttt{OPT}}^R\ge \frac{C^{\textit{BEST}}_{in}-1}{(S_r+S_p)R}.
\end{split}
\end{equation}

Let $(S_r-S_p)(1+\floor{\log{R}})$ be $\mathcal{M}$. We can obtain the following inequalities from Lemmas~\ref{lemma:upper_alg},~\ref{lemma:arbitrary}, and~\ref{lemma:best} as follows.
\begin{equation}~\label{eqn:main_proof_1}
\begin{split}
\texttt{ALG}&\le\alpha C^{\textit{ALG}}_{out}+a\le\beta C^{\textit{ALG}}_{in}+b\le\gamma C^{\textit{BEST}}_{in}+b,
\end{split}
\end{equation}
where $\alpha=\frac{2}{\mathcal{M}}$, $a=\frac{2d_{max}\floor{\log{R}}}{\mathcal{M}}$, $\beta=\frac{6}{\mathcal{M}}$, $b=\frac{12}{\mathcal{M}}+\frac{2d_{max}\floor{\log{R}}}{\mathcal{M}}$, and $\gamma=\frac{36}{\mathcal{M}}$.

Substituting Equation (\ref{eqn:proof_competitive_arbitrary}) into the last inequality of Equation (\ref{eqn:main_proof_1}) and using $\frac{d_{max}}{S_r+S_p}\le {\texttt{OPT}}^R$, we have:
\begin{equation}~\label{eqn:main_proof_2}
\begin{split}
\texttt{ALG}&\le\frac{36(S_r+S_p)R}{\mathcal{M}} {\texttt{OPT}}^R+\frac{48+2d_{max}\floor{\log{R}}}{\mathcal{M}},\\
&\le\frac{2(S_r+S_p)(18R+\floor{\log{R}})}{\mathcal{M}}{\texttt{OPT}}^R+\frac{48}{\mathcal{M}}.
\end{split}
\end{equation}
\end{proof}

\section{Simulations}\label{sec:sim}

\begin{figure*}[hbt!]
\centering
\subfigure[Example of the randomly generated ROI over grid cells. The red dot represents the starting vertex for robots.]{\includegraphics[width=0.35\columnwidth]{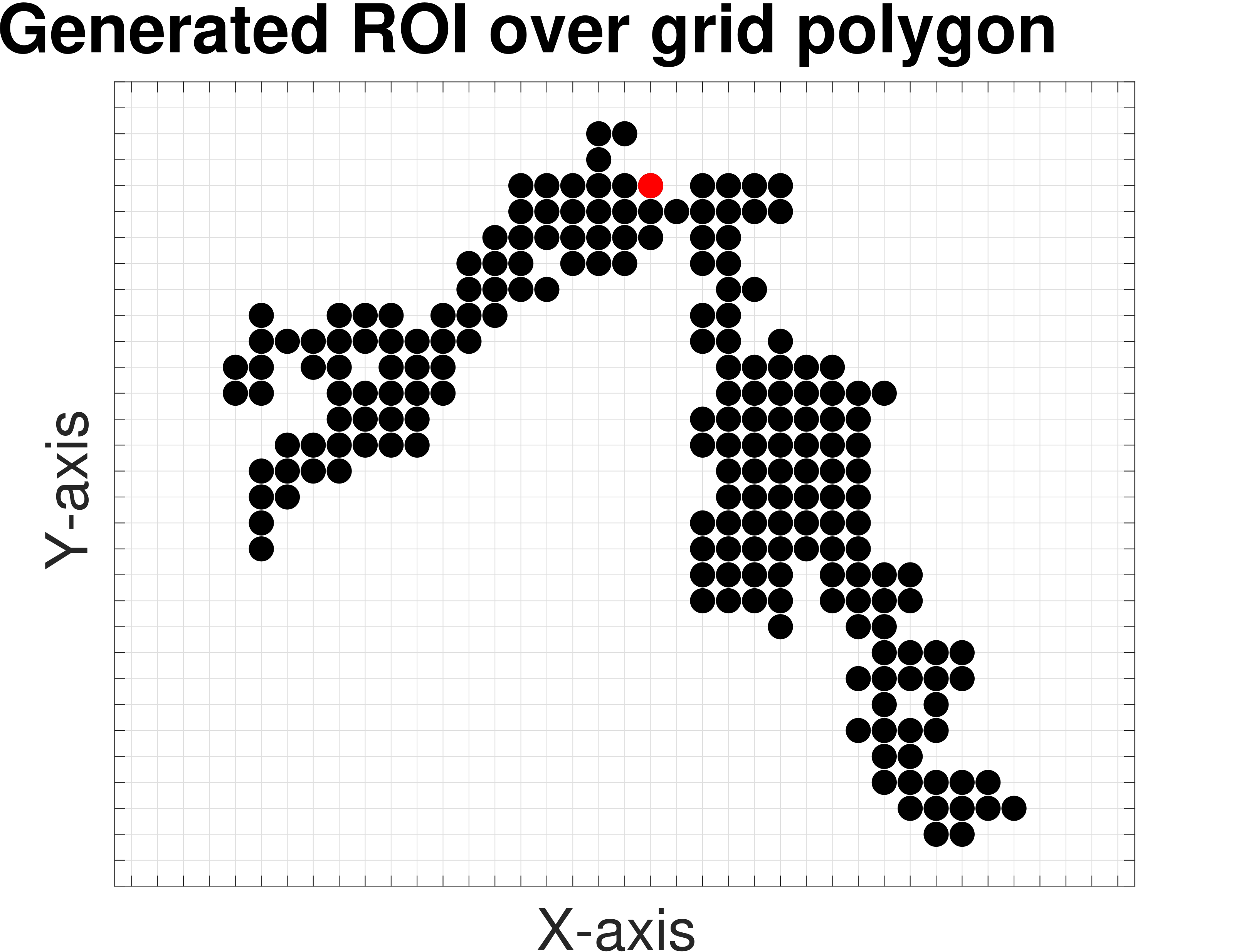}}
\qquad
\subfigure[Plot of the cost when changing the number of ROI cells.]{\includegraphics[width=0.45\columnwidth]{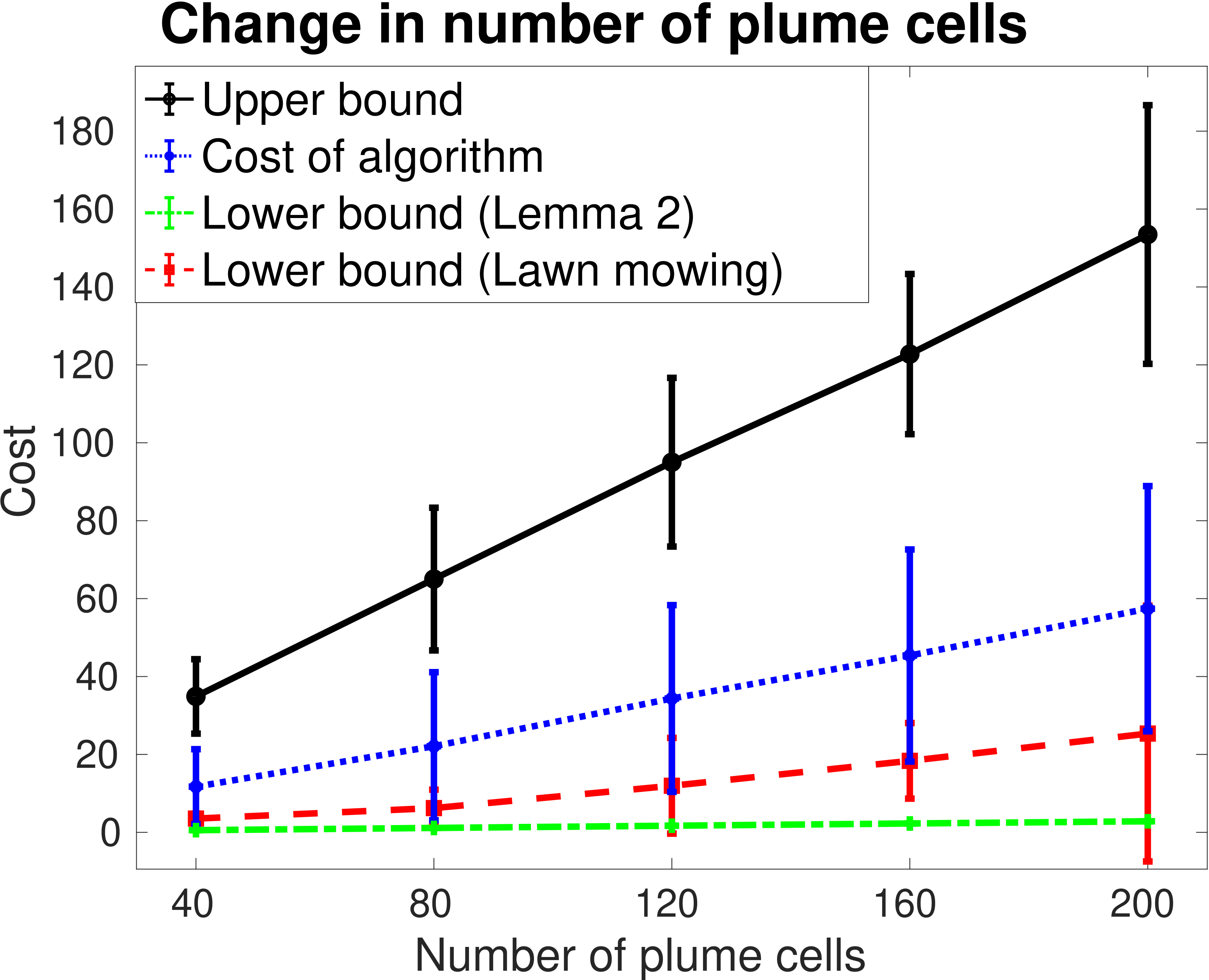}}
\subfigure[Plot of the cost when changing the number of robots.]{\includegraphics[width=0.45\columnwidth]{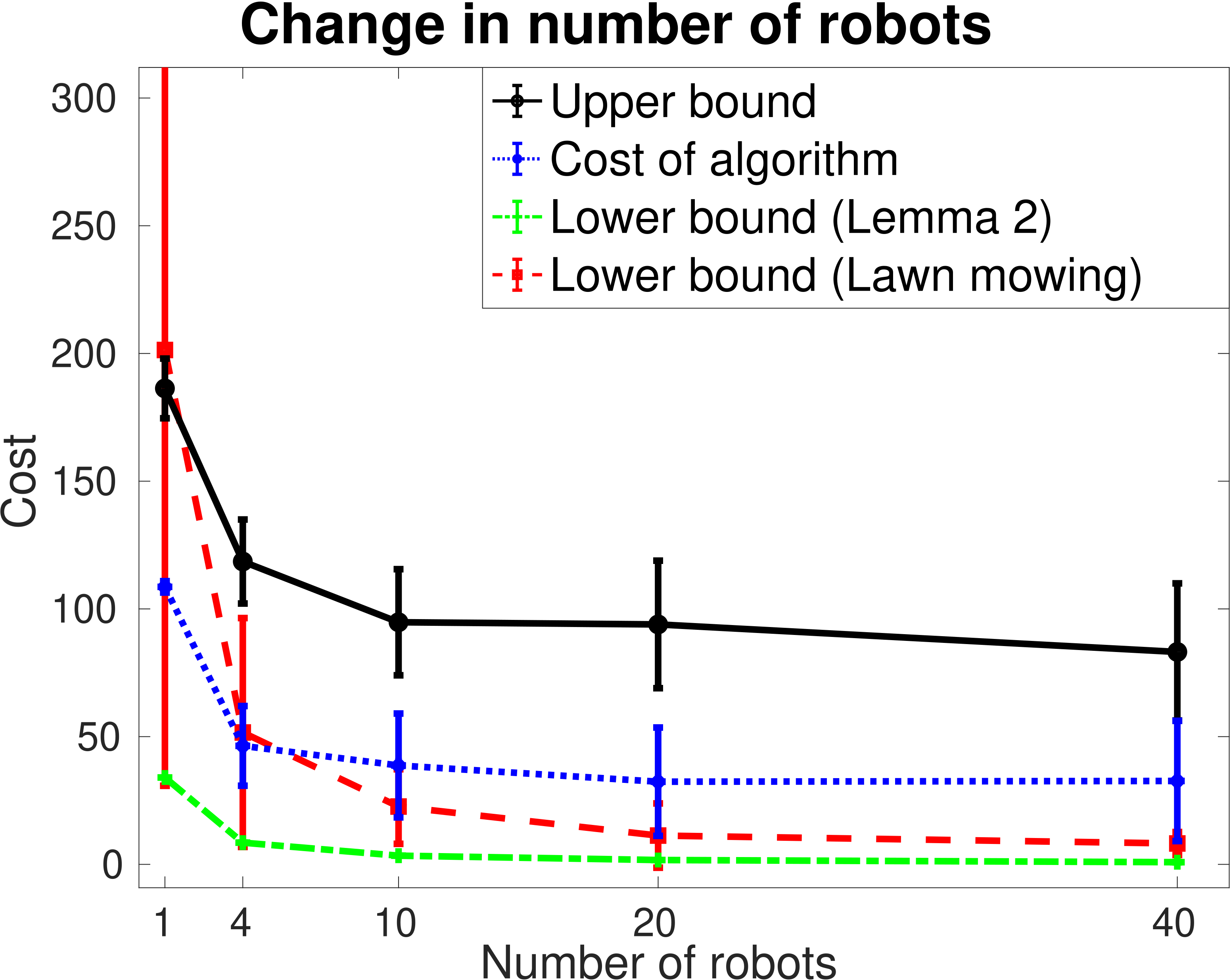}}
\qquad
\subfigure[Plot of the cost when changing the speed ratio between the robot and the ROI. 
]{\includegraphics[width=0.45\columnwidth]{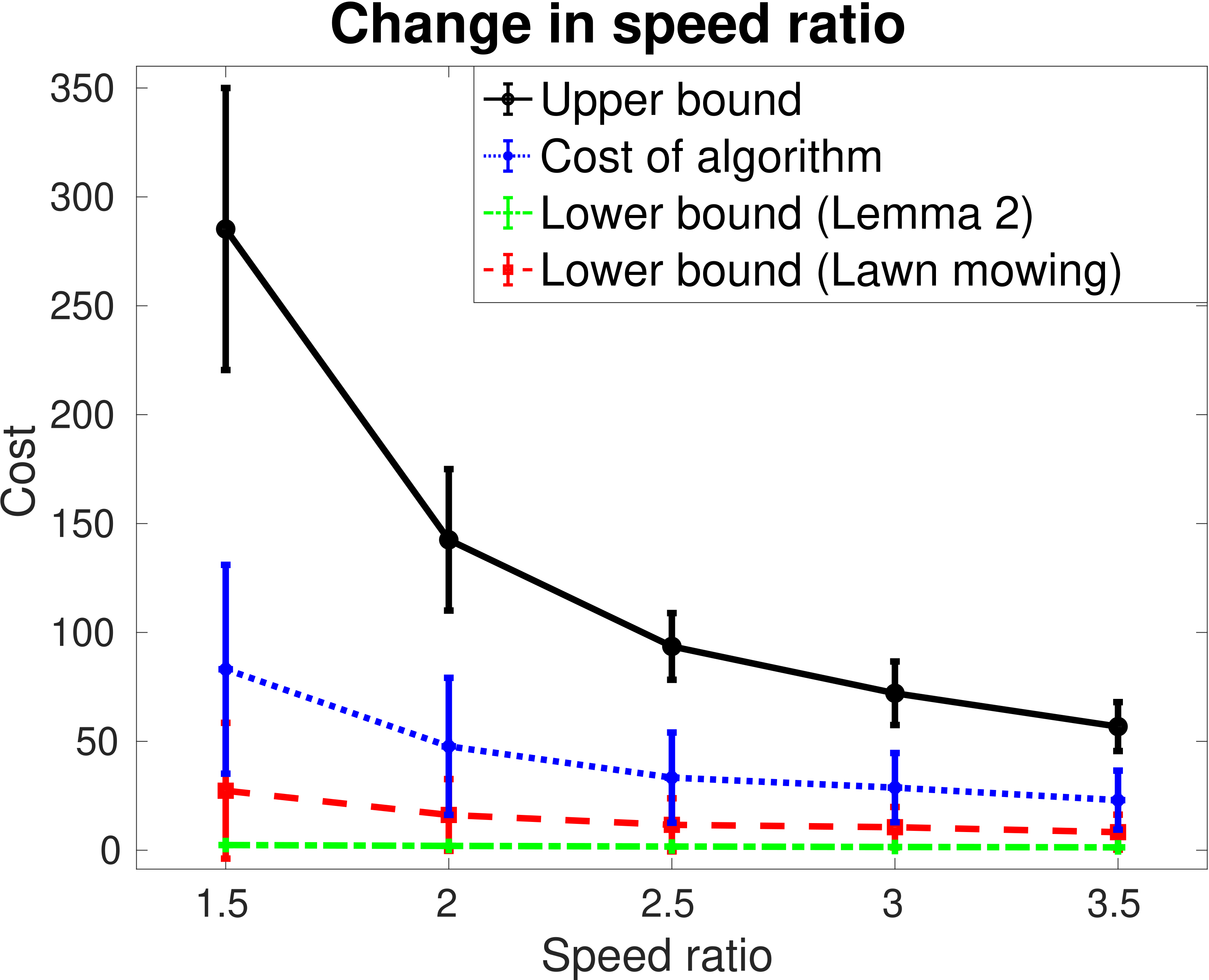}}
\caption{Simulation results. We fixed the number of ROI cells, the number of robots, the speed ratio as $120$, $20$, $2.5$, respectively, when the corresponding variable was not a subject to be changed. We ran $100$ trials for each case. Each case is plotted as mean, maximum and minimum values from $100$ trials.}
\label{fig:simulation}
\end{figure*}

We empirically evaluated our algorithm using MATLAB simulations. Specifically, we verified the performance of the proposed recursive DFS for the grid map approximation of the ROI (Theorem~\ref{theorem:competitive_grid}). 

We randomly generated a set of ROI grid maps. We randomly chose one of four directions (\ie, north, south, east, and west) for the direction of translation of the ROI. The assumption that the moving direction of the ROI is known a priori enables robots to align the axis of the grid map with that direction although robots still do not know about best approximation for grid map. Figure~\ref{fig:simulation} (a) shows an example of the generated ROI that consists of $200$ cells. 

We measured the cost of our algorithm as well as the upper and lower bounds by changing the number of ROI cells, the number of robots, and the speed ratio between the robot and the ROI.
The lower bound in our analysis is given in Lemma~\ref{lemma:optimal}. In deriving the lower bound, we assume that the robots always travel opposite to the direction of translation of the ROI, thereby yielding the lowest possible time. In practice, the robots will not always travel in this best possible direction. Therefore, we find another lower bound using a baseline lawn mower algorithm. We assume that this baseline algorithm knows the tightest rectangle that is guaranteed to contain the ROI. The axis of the rectangle is aligned with the direction of translation of the ROI. Given $R$ robots, we split this rectangle into $R$ smaller ones. We can split this rectangle either along its length or its breadth. The exploration time will be different in each case. We find the time required to cover each smaller rectangle using a lawn mower strategy in both cases and take the lower one. We also ignore the time required for the robots to go from the starting position to the smaller rectangles. This is clearly a lower bound for any online strategy that does not know the size of the ROI. Nevertheless, we find that the exploration time for our algorithm is comparable to this lower bound.

Each case was obtained from $100$ randomly generated trials (see Figures~\ref{fig:simulation} (b--d)). Figure~\ref{fig:simulation} (b) shows that the expected exploration time for all cases is proportional to the number of cells in the grid ROI. The difference between the maximum and minimum costs as well as that between the upper bound and the cost of the algorithm also becomes larger as more ROI cells are to be explored. The cost of our algorithm is closer to the lower bound than the upper bound which becomes even more pronounced as the number of cells increases.

Figure~\ref{fig:simulation} (c) plots the exploration time when changing the number of robots. Not surprisingly, the exploration time of our algorithm and the lower bound decrease as the number of robots increases. The upper bound depends on $d_{max}$. As a result, the upper bound does not decrease as much as the other curves since $d_{max}$ can be high for randomly generated maps. Regardless, we find that our algorithm performs better empirically than what is given by the theoretical upper bound.

Figure~\ref{fig:simulation} (d) shows the exploration time when changing the speed ratio between the robot and the ROI, \ie, $\frac{S_r}{S_p}$. The exploration time for our algorithm and the upper bound decrease as the speed ratio increases. We also observe that the difference between upper and lower bounds decreases as the speed ratio increases.

The simulation results verify the theoretical upper and lower bounds determined by our analysis. In addition, they demonstrate that the practical performance of our algorithm is better than that indicated by the upper bounds. We observe that the practical performance is closer to the lower bound than the upper bound.

\section{Field Experiment}\label{sec:experiment}

In this section, we conduct proof-of-concept experiments using a single UAV equipped with a downward-facing camera to  monitor a stationary region of unknown size and shape. 
The goal of the field experiment is to show how the proposed algorithm can be implemented on an actual robot and can be deployed in the field for online exploration although limited to the single robot case.
In a practical implementation, there are a number of design choices that must be made (\eg, what altitude to fly? how to convert the camera images into cell measurements? how to deal with erroneous sensor measurements?). In this section, we answer these questions in the context of our system. 


Our environment to be mapped is a $92m\times 21m$ long runway which serves  as a proxy of the ROI.
Figure~\ref{fig:setting} shows hardware details of the UAV and the snapshot of the environment. The UAV has ODROID-XU4 single-board computer which runs Ubuntu 16.04 with ROS Kinetic~\cite{quigley2009ros}. The onboard software controls the UAV, communicates with GoPro HERO4 camera over WiFi to read sensor information, and detects the runway.

\begin{figure}[thpb]
\centering
\subfigure[UAV platform.]{\includegraphics[width=0.35\columnwidth]{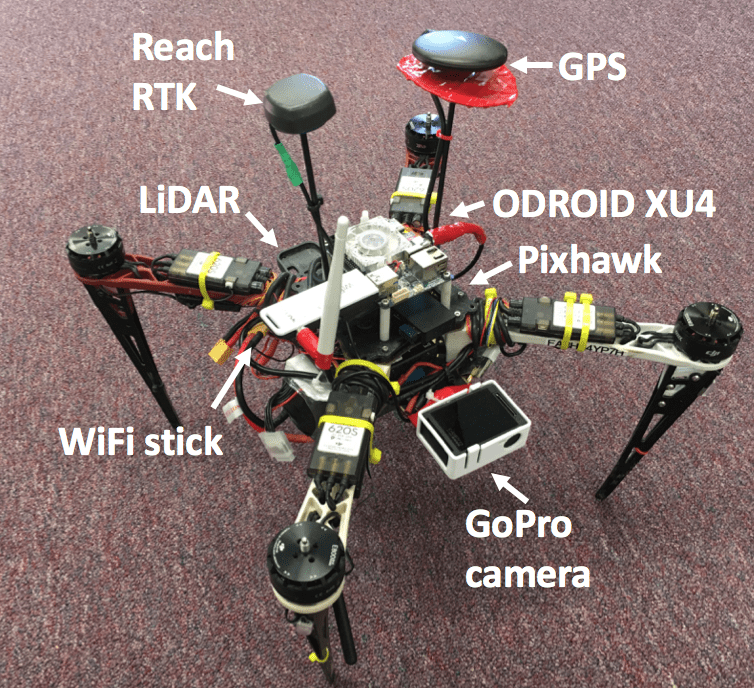}}
\qquad
\subfigure[Runway ($92m\times 21m$) to be explored.]{\includegraphics[width=0.45\columnwidth]{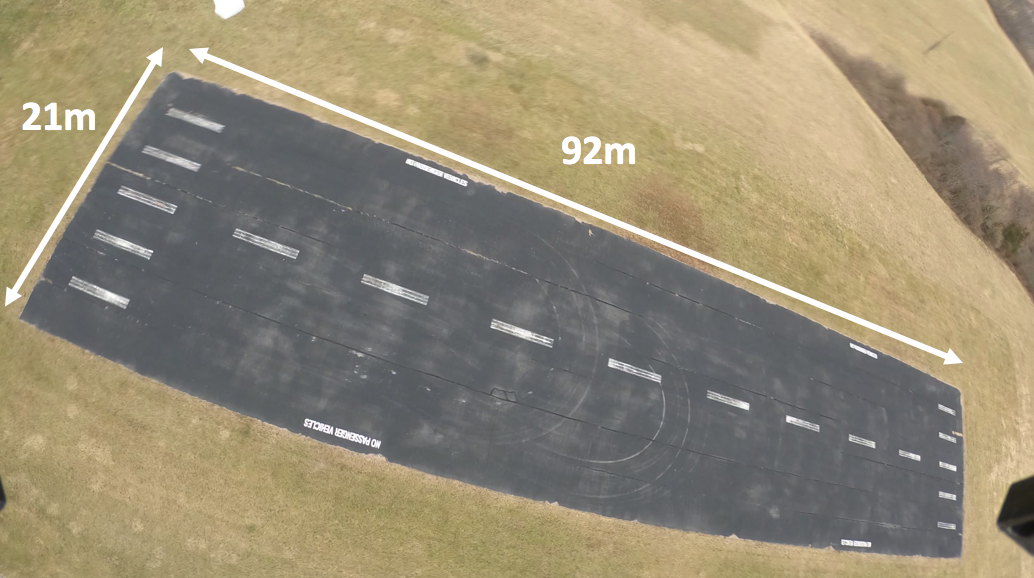}}
\caption{Experimental setting.}
\label{fig:setting}
\end{figure}

Our planning strategy consists of two modes: (1) lawn mowing and (2) the single-robot version of Algorithm~\ref{alg:dfs}. The input to the lawn-mowing mode is a bounding box that is guaranteed to contain at least some part of the runway. In the lawn-mowing mode, the UAV sweeps the bounding box. As soon as the UAV observes a part of the runway, the UAV switches to the recursive DFS algorithm.  

Once in the recursive DFS mode, we discretize the environment into grid cells. The grid is aligned with the UTM coordinates~\cite{wiki:xxx}. The origin of the grid is placed at the starting location of the UAV. Note that the grid is not aligned with the runway (Figure~\ref{fig:setting} (b)) and the location, shape, and size of the runway are not given to the UAV a priori. 

Each cell is of size $4m\times 4m$. We use the onboard GoPro images to determine if a cell corresponds to the runway. Each image is divided into $3\times 3$ regions (Figure~\ref{fig:sensing_model}). The size of a grid map cell ($4m\times 4m$) corresponds to the footprint of the center region in the image when flying at an altitude of $12m$. The center and the top, left, bottom, and right regions in the image (refer to Figure~\ref{fig:sensing_model}) are used to determine if the current cell and its four neighbors in the grid map contain the runway or not.

\begin{figure}[thpb]
\centering
\includegraphics[width=0.60\columnwidth]{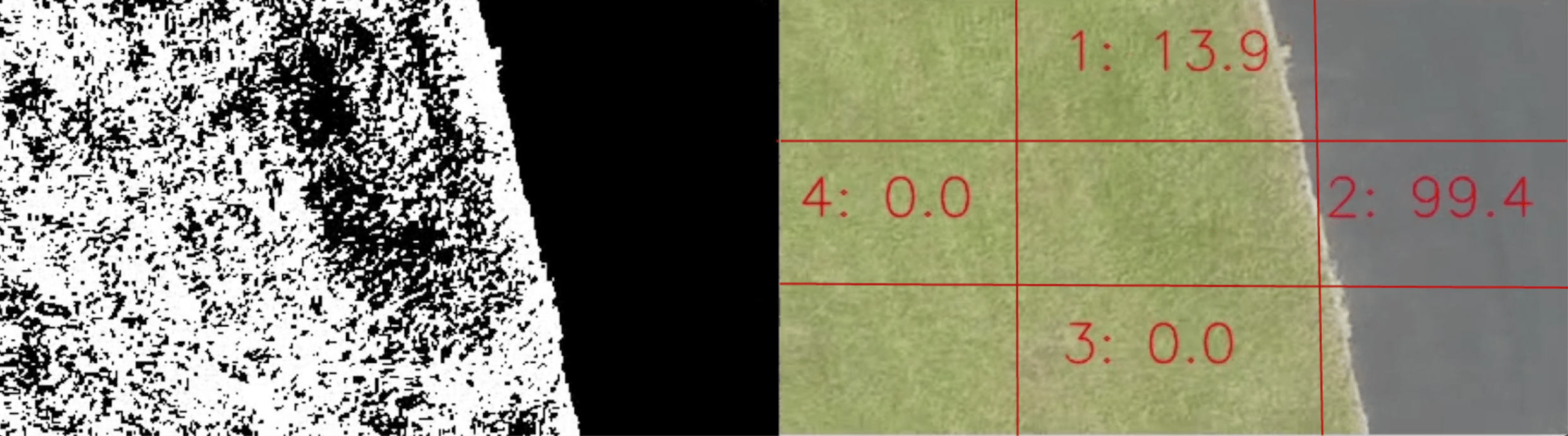}
\caption{Example of our sensing model using a single image that contains both the runway and the grass region. The left image is the thresholded image. The right image shows the detection result indicating the percentage of black pixel values printed on the grid cells (colored in red in four neighboring cells).}
\label{fig:sensing_model}
\end{figure}

Each image is first converted into grayscale and thresholded (at the intensity value of $150$). Then, the thresholded image is dilated ($7$ times) so that the gaps in the grass region are filled in (refer the left image in Figure~\ref{fig:sensing_model}). Consequently, the entire grass region becomes filled with white pixels and the runway region mostly comprises black pixels. Then the percentage of black pixels in the individual regions is calculated,  for each highlighted neighbouring region in the figure. Finally, we produce a binary classification result based on if this percentage is above or below a  threshold. 

Our classifier may not give a correct detection result due to a number of reasons. First, it may produce false positives and false negatives. Second, the detection result is sensitive to the intensity threshold value we set. Third, even if the UAV is on top of the current cell, the camera may point to a wrong cell due to pitch and roll used to counter wind disturbance and imperfect flight controller. Lastly, the change in the sunlight condition might produce noisy measurements. Therefore, instead of relying on a single image, we use five images per cell. If three of these images mark a cell as containing the runway, we treat it as a positive detection. When mapping a region, we care more about the completeness of exploration rather than the efficiency. Therefore, we make the following conservative design choices. If a cell is marked as containing the runway, it will never be reversed later on, even if a future measurement taken from some other location yields the opposite detection. On the other hand, if a cell is marked as not containing the runway and a subsequent measurement suggests otherwise, we will mark it as a positive runway detection. 

The measured flight time of the UAV with a single battery is approximately $12$ minutes, which is not enough for finishing the exploration task. Therefore, we designed our software in a way to keep track of its previous computation even after replacing the battery. To do that, the software stores the information of the generated tree (\ie, the status of vertices and the tree structure) and feeds that to the next flight in order for the UAV to start from the vertex visited last in the previous flight.


\begin{figure*}[hbt!]
\centering
\subfigure[Snapshot at time $2$ minutes. ]{\includegraphics[width=0.37\columnwidth]{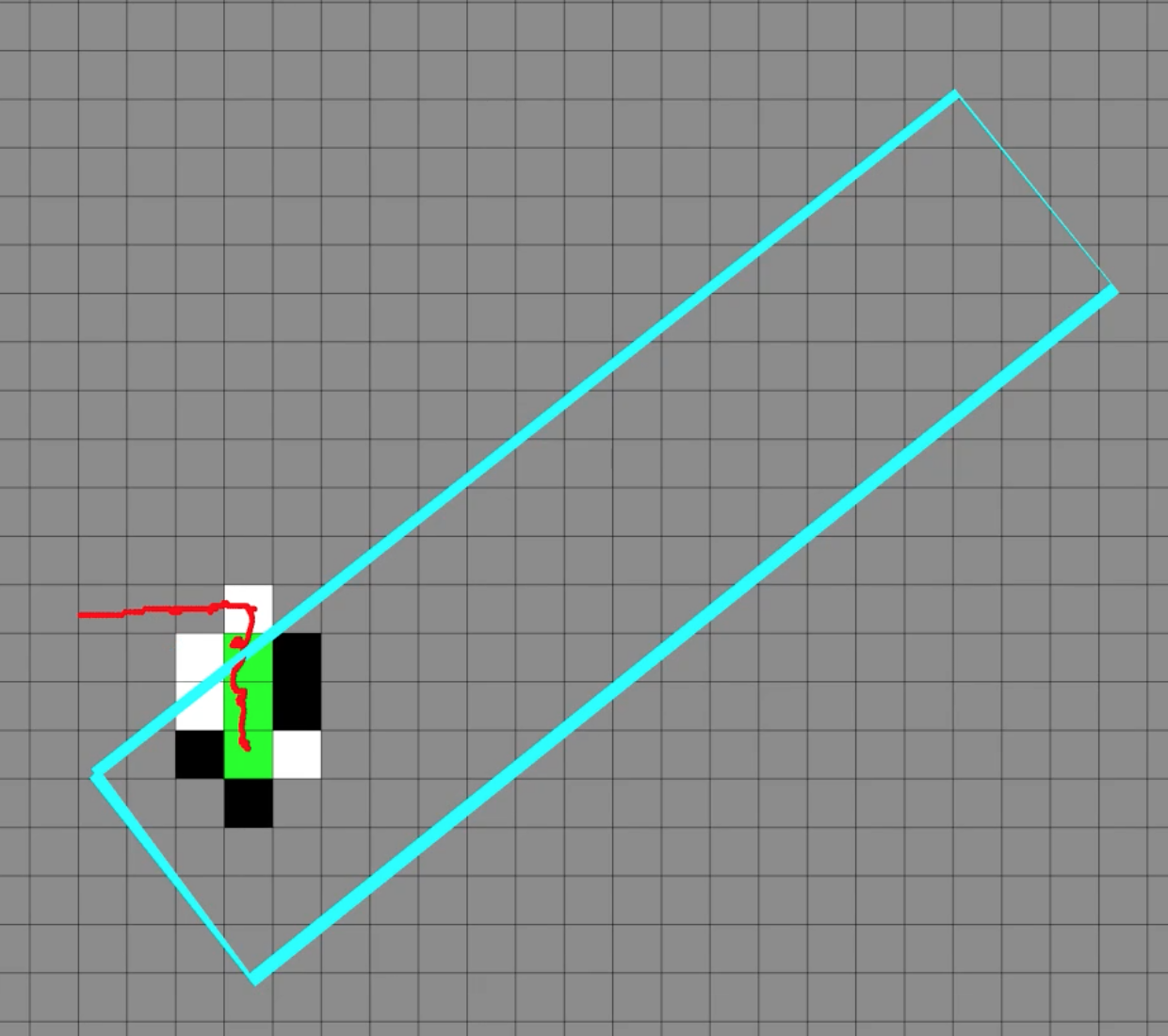}}
\subfigure[Snapshot at time $26$ minutes and $40$ seconds.]{\includegraphics[width=0.37\columnwidth]{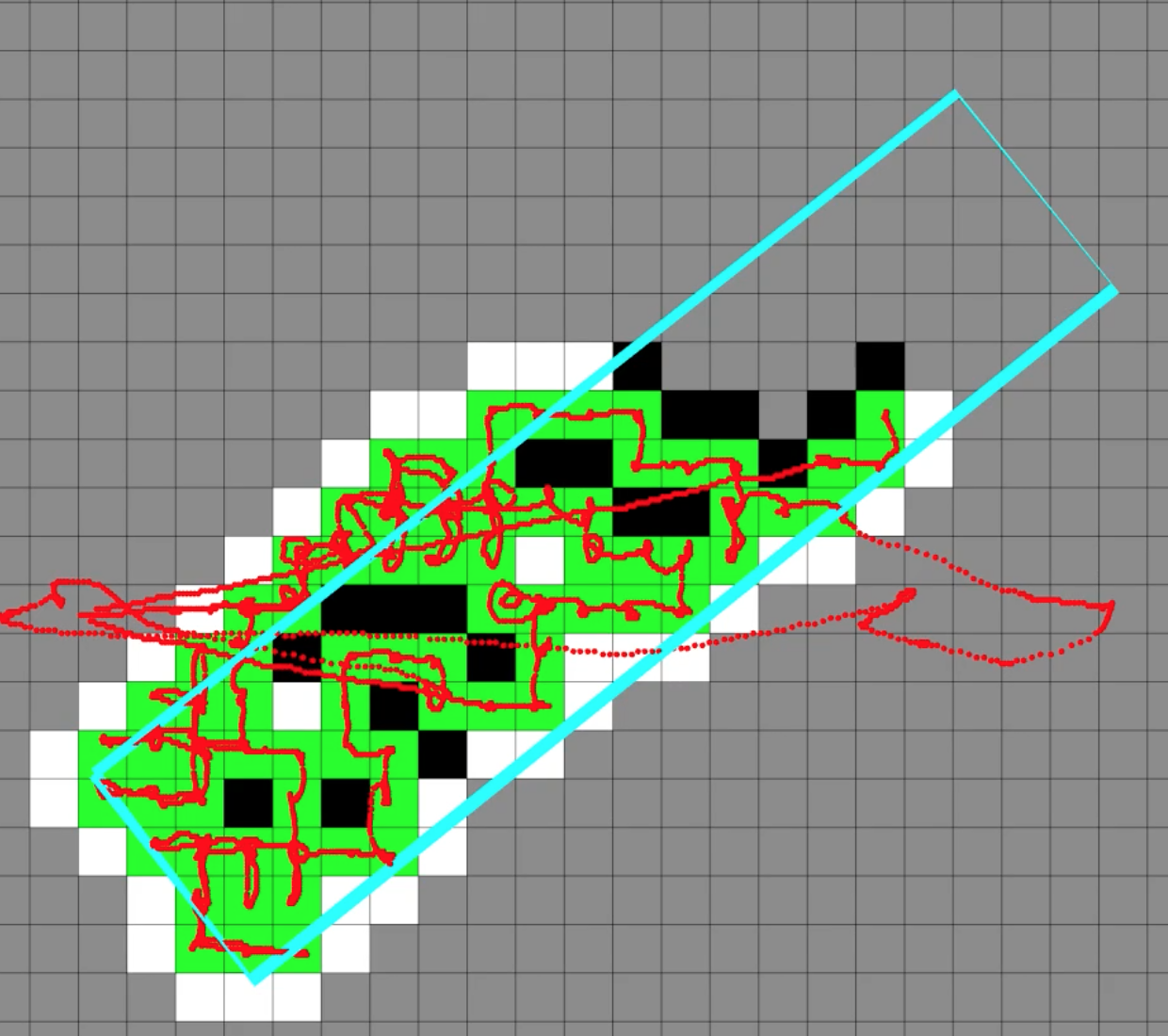}}
\subfigure[Snapshot at time $1$ hour and $13$ minutes.]{\includegraphics[width=0.37\columnwidth]{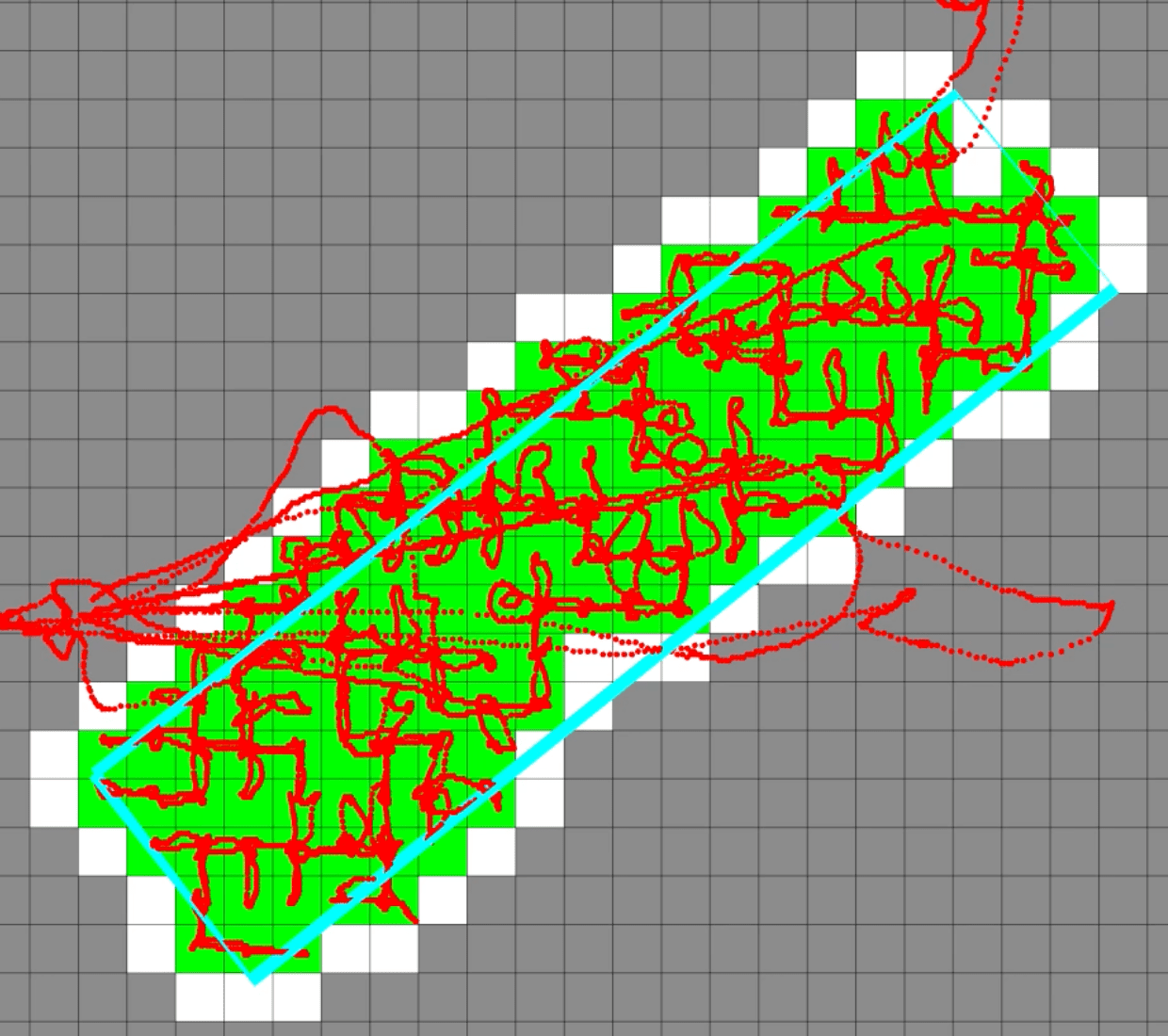}}
\caption{Experimental result. The blue square line represents the ground truth of the runway to be explored that is unknown initially. The red line denotes the trajectory of the UAV. The size of each cell is $4m\times 4m$ and gray, white, black and green colors represent unknown, non-runway, unexplored runway and explored cells, respectively. The total flight time taken to completely explore the entire runway was $1$ hour and $13$ minutes, consisting of six battery replacements. The video is available at: \url{https://youtu.be/ZTpPb0C4Lk0}.}
\label{fig:experiment}
\end{figure*}

Figure~\ref{fig:experiment} shows the result of the runway exploration experiment. We flew the UAV at a nominal speed of $1m/s$ and the ambient wind speed was approximately $3.6m/s$. The total flight time was 1 hour and 13 minutes and consisted of six battery replacements. The final tree contained 143 vertices.

\begin{figure}
\centering
  \includegraphics[width=0.40\columnwidth]{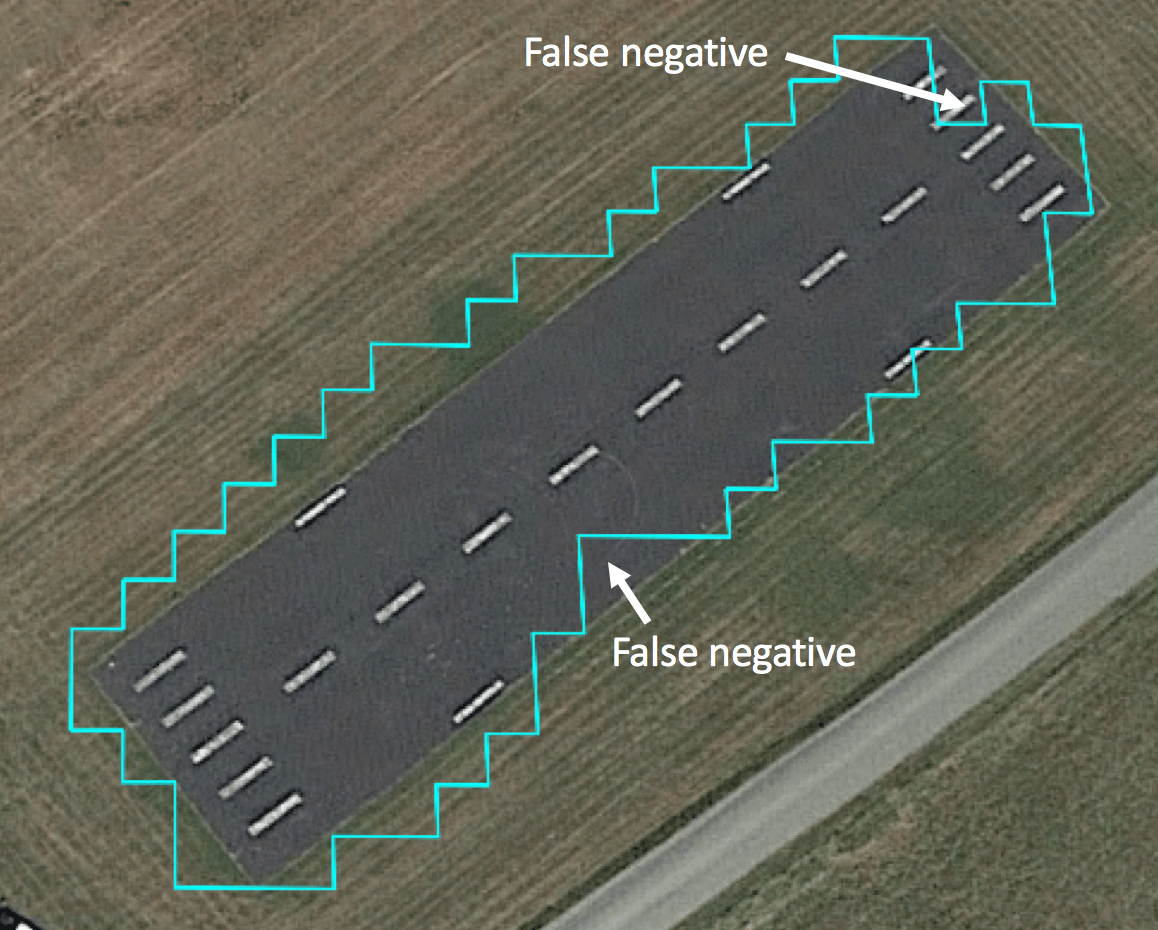}
  \caption{Resultant boundary of the runway region (\ie, the boundary of the grid map) plotted on Google Earth colored in blue.}
  \label{fig:detection_result}
\end{figure}

The final grid map overlaid on Google Earth is shown in Figure~\ref{fig:detection_result}. The percentage of the area of intersection between the ground-truth and the final map, normalized by the area of their union, was $75.7\%$. 

We also compute the false-positive and false-negative detection rates. 
Out of the total 483 detections, 27 were false positives and 53 were false negatives. Note that all but two cells that gave false-negative detections, eventually gave a positive detection. As a result, the final map (Figure~\ref{fig:detection_result}) has only two cells that are incorrectly mapped as not being part of the runway. The cells that are just outside the boundary, however, are incorrectly mapped as being part of the runway. This is likely because of our conservative exploration choice of marking a cell as containing the runway even if a single detection is positive.

In summary, we present how our algorithm handles real-world issues and can be implemented on actual robots. We show the efficacy of the proposed scheme through field experiments.


\section{Conclusion}\label{sec:conc}
We propose a recursive DFS algorithm for a team of aerial robots to explore a translating ROI without knowing its size and shape. We present two approaches for the given problem where the first approximates the ROI to map to the grid whereas the second considers any arbitrary shape of ROI as long as it is \emph{fat}. Both approaches are competitive with respect to the optimal algorithm.
We demonstrate the performance of our algorithm through proof-of-concept deployment to map a stationary ROI. 

One of the practical concerns not modeled by our system is that robots, especially UAVs, have limited battery lifetime. As such, we would like to devise algorithms that can map the ROI subject to the limited battery lifetime constraint. In particular, in our formulation, we restrict the UAV to fly at a fixed altitude. However, one may be able to extend the coverage range by flying at higher altitudes. An interesting and relevant extension of this work would be to plan in 3D space as opposed to 2D.



\bibliographystyle{IEEEtran}
\bibliography{IEEEabrv,yoon_refs}

\begin{thebibliography}{10}
\providecommand{\url}[1]{#1}
\csname url@rmstyle\endcsname
\providecommand{\newblock}{\relax}
\providecommand{\bibinfo}[2]{#2}
\providecommand\BIBentrySTDinterwordspacing{\spaceskip=0pt\relax}
\providecommand\BIBentryALTinterwordstretchfactor{4}
\providecommand\BIBentryALTinterwordspacing{\spaceskip=\fontdimen2\font plus
\BIBentryALTinterwordstretchfactor\fontdimen3\font minus
  \fontdimen4\font\relax}
\providecommand\BIBforeignlanguage[2]{{%
\expandafter\ifx\csname l@#1\endcsname\relax
\typeout{** WARNING: IEEEtran.bst: No hyphenation pattern has been}%
\typeout{** loaded for the language `#1'. Using the pattern for}%
\typeout{** the default language instead.}%
\else
\language=\csname l@#1\endcsname
\fi
#2}}

\bibitem{fernandez2015uav}
J.~Fernandez~Galarreta, N.~Kerle, and M.~Gerke, ``Uav-based urban structural
  damage assessment using object-based image analysis and semantic reasoning,''
  \emph{Natural hazards and earth system sciences}, vol.~15, no.~6, pp.
  1087--1101, 2015.

\bibitem{stachniss2003exploring}
C.~Stachniss and W.~Burgard, ``Exploring unknown environments with mobile
  robots using coverage maps,'' in \emph{IJCAI}, vol. 2003, 2003, pp.
  1127--1134.

\bibitem{masehian2017cooperative}
E.~Masehian, M.~Jannati, and T.~Hekmatfar, ``Cooperative mapping of unknown
  environments by multiple heterogeneous mobile robots with limited sensing,''
  \emph{Robotics and Autonomous Systems}, vol.~87, pp. 188--218, 2017.

\bibitem{song2020care}
J.~Song and S.~Gupta, ``Care: Cooperative autonomy for resilience and
  efficiency of robot teams for complete coverage of unknown environments under
  robot failures,'' \emph{Autonomous Robots}, vol.~44, no.~3, pp. 647--671,
  2020.

\bibitem{brass2011multirobot}
P.~Brass, F.~Cabrera-Mora, A.~Gasparri, and J.~Xiao, ``Multirobot tree and
  graph exploration,'' \emph{IEEE Transactions on Robotics}, vol.~27, no.~4,
  pp. 707--717, 2011.

\bibitem{julia2012comparison}
M.~Juli{\'a}, A.~Gil, and O.~Reinoso, ``A comparison of path planning
  strategies for autonomous exploration and mapping of unknown environments,''
  \emph{Autonomous Robots}, vol.~33, no.~4, pp. 427--444, 2012.

\bibitem{nuske2015autonomous}
S.~Nuske, S.~Choudhury, S.~Jain, A.~Chambers, L.~Yoder, S.~Scherer,
  L.~Chamberlain, H.~Cover, and S.~Singh, ``Autonomous exploration and motion
  planning for an unmanned aerial vehicle navigating rivers,'' \emph{Journal of
  Field Robotics}, vol.~32, no.~8, pp. 1141--1162, 2015.

\bibitem{girdhar2014autonomous}
Y.~Girdhar, P.~Giguere, and G.~Dudek, ``Autonomous adaptive exploration using
  realtime online spatiotemporal topic modeling,'' \emph{The International
  Journal of Robotics Research}, vol.~33, no.~4, pp. 645--657, 2014.

\bibitem{icking2005exploring}
C.~Icking, T.~Kamphans, R.~Klein, and E.~Langetepe, ``Exploring simple grid
  polygons,'' in \emph{International Computing and Combinatorics
  Conference}.\hskip 1em plus 0.5em minus 0.4em\relax Springer, 2005, pp.
  524--533.

\bibitem{kolenderska2009improved}
A.~Kolenderska, A.~Kosowski, M.~Ma{\l}afiejski, and P.~{\.Z}yli{\'n}ski, ``An
  improved strategy for exploring a grid polygon,'' in \emph{International
  Colloquium on Structural Information and Communication Complexity}.\hskip 1em
  plus 0.5em minus 0.4em\relax Springer, 2009, pp. 222--236.

\bibitem{fraigniaud2006collective}
P.~Fraigniaud, L.~Gasieniec, D.~R. Kowalski, and A.~Pelc, ``Collective tree
  exploration,'' \emph{Networks: An International Journal}, vol.~48, no.~3, pp.
  166--177, 2006.

\bibitem{higashikawa2014online}
Y.~Higashikawa, N.~Katoh, S.~Langerman, and S.-i. Tanigawa, ``Online graph
  exploration algorithms for cycles and trees by multiple searchers,''
  \emph{Journal of Combinatorial Optimization}, vol.~28, no.~2, pp. 480--495,
  2014.

\bibitem{tokekar2013tracking}
P.~Tokekar, E.~Branson, J.~Vander~Hook, and V.~Isler, ``Tracking aquatic
  invaders: Autonomous robots for monitoring invasive fish,'' \emph{IEEE
  Robotics \& Automation Magazine}, vol.~20, no.~3, pp. 33--41, 2013.

\bibitem{borodin2005online}
A.~Borodin and R.~El-Yaniv, \emph{Online computation and competitive
  analysis}.\hskip 1em plus 0.5em minus 0.4em\relax cambridge university press,
  2005.

\bibitem{sung2019competitive}
Y.~Sung and P.~Tokekar, ``A competitive algorithm for online multi-robot
  exploration of a translating plume,'' in \emph{2019 International Conference
  on Robotics and Automation (ICRA)}.\hskip 1em plus 0.5em minus 0.4em\relax
  IEEE, 2019, pp. 3391--3397.

\bibitem{tokekar2016sensor}
P.~Tokekar, J.~Vander~Hook, D.~Mulla, and V.~Isler, ``Sensor planning for a
  symbiotic uav and ugv system for precision agriculture,'' \emph{IEEE
  Transactions on Robotics}, vol.~32, no.~6, pp. 1498--1511, 2016.

\bibitem{das2015devices}
J.~Das, G.~Cross, C.~Qu, A.~Makineni, P.~Tokekar, Y.~Mulgaonkar, and V.~Kumar,
  ``Devices, systems, and methods for automated monitoring enabling precision
  agriculture,'' in \emph{Automation Science and Engineering (CASE), 2015 IEEE
  International Conference on}.\hskip 1em plus 0.5em minus 0.4em\relax IEEE,
  2015, pp. 462--469.

\bibitem{tokekar2010robotic}
P.~Tokekar, D.~Bhadauria, A.~Studenski, and V.~Isler, ``A robotic system for
  monitoring carp in minnesota lakes,'' \emph{Journal of Field Robotics},
  vol.~27, no.~6, pp. 779--789, 2010.

\bibitem{plonski2017environment}
P.~A. Plonski, J.~Vander~Hook, C.~Peng, N.~Noori, and V.~Isler, ``Environment
  exploration in sensing automation for habitat monitoring,'' \emph{IEEE
  Transactions on Automation Science and Engineering}, vol.~14, no.~1, pp.
  25--38, 2017.

\bibitem{ishida2012chemical}
H.~Ishida, Y.~Wada, and H.~Matsukura, ``Chemical sensing in robotic
  applications: A review,'' \emph{IEEE Sensors Journal}, vol.~12, no.~11, pp.
  3163--3173, 2012.

\bibitem{lochmatter2013plume}
T.~Lochmatter, E.~A. G{\"o}l, I.~Navarro, and A.~Martinoli, ``A plume tracking
  algorithm based on crosswind formations,'' in \emph{Distributed Autonomous
  Robotic Systems}.\hskip 1em plus 0.5em minus 0.4em\relax Springer, 2013, pp.
  91--102.

\bibitem{fahad2017robotic}
M.~Fahad, Y.~Guo, B.~Bingham, K.~Krasnosky, L.~Fitzpatrick, and F.~A. Sanabria,
  ``Robotic experiments to evaluate ocean plume characteristics and
  structure,'' in \emph{Intelligent Robots and Systems (IROS), 2017 IEEE/RSJ
  International Conference on}.\hskip 1em plus 0.5em minus 0.4em\relax IEEE,
  2017, pp. 6098--6104.

\bibitem{dunbabin2012robots}
M.~Dunbabin and L.~Marques, ``Robots for environmental monitoring: Significant
  advancements and applications,'' \emph{IEEE Robotics \& Automation Magazine},
  vol.~19, no.~1, pp. 24--39, 2012.

\bibitem{galceran2013survey}
E.~Galceran and M.~Carreras, ``A survey on coverage path planning for
  robotics,'' \emph{Robotics and Autonomous systems}, vol.~61, no.~12, pp.
  1258--1276, 2013.

\bibitem{arora2017randomized}
S.~Arora and S.~Scherer, ``Randomized algorithm for informative path planning
  with budget constraints,'' in \emph{2017 IEEE International Conference on
  Robotics and Automation (ICRA)}.\hskip 1em plus 0.5em minus 0.4em\relax IEEE,
  2017, pp. 4997--5004.

\bibitem{popovic2020informative}
M.~Popovi{\'c}, T.~Vidal-Calleja, G.~Hitz, J.~J. Chung, I.~Sa, R.~Siegwart, and
  J.~Nieto, ``An informative path planning framework for uav-based terrain
  monitoring,'' \emph{Autonomous Robots}, pp. 1--23, 2020.

\bibitem{corah2019distributed}
M.~Corah and N.~Michael, ``Distributed matroid-constrained submodular
  maximization for multi-robot exploration: Theory and practice,''
  \emph{Autonomous Robots}, vol.~43, no.~2, pp. 485--501, 2019.

\bibitem{singh2009efficient}
A.~Singh, A.~Krause, C.~Guestrin, and W.~J. Kaiser, ``Efficient informative
  sensing using multiple robots,'' \emph{Journal of Artificial Intelligence
  Research}, vol.~34, pp. 707--755, 2009.

\bibitem{hitz2014fully}
G.~Hitz, A.~Gotovos, M.-{\'E}. Garneau, C.~Pradalier, A.~Krause, R.~Y.
  Siegwart, \emph{et~al.}, ``Fully autonomous focused exploration for robotic
  environmental monitoring,'' in \emph{2014 IEEE International Conference on
  Robotics and Automation (ICRA)}.\hskip 1em plus 0.5em minus 0.4em\relax IEEE,
  2014, pp. 2658--2664.

\bibitem{sim2009autonomous}
R.~Sim and J.~J. Little, ``Autonomous vision-based robotic exploration and
  mapping using hybrid maps and particle filters,'' \emph{Image and Vision
  Computing}, vol.~27, no. 1-2, pp. 167--177, 2009.

\bibitem{cesare2015multi}
K.~Cesare, R.~Skeele, S.-H. Yoo, Y.~Zhang, and G.~Hollinger, ``Multi-uav
  exploration with limited communication and battery,'' in \emph{Robotics and
  Automation (ICRA), 2015 IEEE International Conference on}.\hskip 1em plus
  0.5em minus 0.4em\relax IEEE, 2015, pp. 2230--2235.

\bibitem{bender2002power}
M.~A. Bender, A.~Fern{\'a}ndez, D.~Ron, A.~Sahai, and S.~Vadhan, ``The power of
  a pebble: Exploring and mapping directed graphs,'' \emph{Information and
  computation}, vol. 176, no.~1, pp. 1--21, 2002.

\bibitem{das2007map}
S.~Das, P.~Flocchini, S.~Kutten, A.~Nayak, and N.~Santoro, ``Map construction
  of unknown graphs by multiple agents,'' \emph{Theoretical Computer Science},
  vol. 385, no. 1-3, pp. 34--48, 2007.

\bibitem{icking2000exploring}
C.~Icking, T.~Kamphans, R.~Klein, and E.~Langetepe, ``Exploring an unknown
  cellular environment.'' in \emph{EuroCG}, 2000, pp. 140--143.

\bibitem{arkin2000approximation}
E.~M. Arkin, S.~P. Fekete, and J.~S. Mitchell, ``Approximation algorithms for
  lawn mowing and milling,'' \emph{Computational Geometry}, vol.~17, no. 1-2,
  pp. 25--50, 2000.

\bibitem{arya2001approximation}
S.~Arya, S.-W. Cheng, and D.~M. Mount, ``Approximation algorithm for
  multiple-tool milling,'' \emph{International Journal of Computational
  Geometry \& Applications}, vol.~11, no.~03, pp. 339--372, 2001.

\bibitem{van1994motion}
A.~F. van~der Stappen and M.~H. Overmars, ``Motion planning amidst fat
  obstacles,'' in \emph{Proceedings of the tenth annual symposium on
  Computational geometry}.\hskip 1em plus 0.5em minus 0.4em\relax ACM, 1994,
  pp. 31--40.

\bibitem{efrat2005complexity}
A.~Efrat, ``The complexity of the union of ($\alpha$,$\beta$)-covered
  objects,'' \emph{SIAM Journal on Computing}, vol.~34, no.~4, pp. 775--787,
  2005.

\bibitem{aloupis2014triangulating}
G.~Aloupis, P.~Bose, V.~Dujmovi{\'c}, C.~Gray, S.~Langerman, and B.~Speckmann,
  ``Triangulating and guarding realistic polygons,'' \emph{Computational
  geometry}, vol.~47, no.~2, pp. 296--306, 2014.

\bibitem{lee2016structured}
S.~K. Lee, S.~P. Fekete, and J.~McLurkin, ``Structured triangulation in
  multi-robot systems: Coverage, patrolling, voronoi partitions, and geodesic
  centers,'' \emph{The International Journal of Robotics Research}, vol.~35,
  no.~10, pp. 1234--1260, 2016.

\bibitem{gabriely2001spanning}
Y.~Gabriely and E.~Rimon, ``Spanning-tree based coverage of continuous areas by
  a mobile robot,'' \emph{Annals of mathematics and artificial intelligence},
  vol.~31, no. 1-4, pp. 77--98, 2001.

\bibitem{klein2015local}
R.~Klein, D.~Kriesel, and E.~Langetepe, ``A local strategy for cleaning
  expanding cellular domains by simple robots,'' \emph{Theoretical Computer
  Science}, vol. 605, pp. 80--94, 2015.

\bibitem{sharma2019optimal}
G.~Sharma, A.~Dutta, and J.-H. Kim, ``Optimal online coverage path planning
  with energy constraints,'' in \emph{Proceedings of the 18th International
  Conference on Autonomous Agents and MultiAgent Systems}, 2019, pp.
  1189--1197.

\bibitem{mahadev2017mapping}
A.~Mahadev, D.~Krupke, S.~P. Fekete, and A.~T. Becker, ``Mapping and coverage
  with a particle swarm controlled by uniform inputs,'' in \emph{2017 IEEE/RSJ
  International Conference on Intelligent Robots and Systems (IROS)}.\hskip 1em
  plus 0.5em minus 0.4em\relax IEEE, 2017, pp. 1097--1104.

\bibitem{dynia2007robots}
M.~Dynia, J.~{\L}opusza{\'N}ski, and C.~Schindelhauer, ``Why robots need
  maps,'' in \emph{International Colloquium on Structural Information and
  Communication Complexity}.\hskip 1em plus 0.5em minus 0.4em\relax Springer,
  2007, pp. 41--50.

\bibitem{preshant2016geometric}
\BIBentryALTinterwordspacing
A.~Preshant, K.~Yu, and P.~Tokekar, ``A geometric approach for multi-robot
  exploration in orthogonal polygons,'' in \emph{Workshop on Algorithmic
  Foundations of Robotics (WAFR)}, 2016. [Online]. Available:
  \url{http://www.wafr.org/papers/WAFR_2016_paper_25.pdf}
\BIBentrySTDinterwordspacing

\bibitem{das2015collaborative}
S.~Das, D.~Dereniowski, and C.~Karousatou, ``Collaborative exploration by
  energy-constrained mobile robots,'' in \emph{International Colloquium on
  Structural Information and Communication Complexity}.\hskip 1em plus 0.5em
  minus 0.4em\relax Springer, 2015, pp. 357--369.

\bibitem{megow2012online}
N.~Megow, K.~Mehlhorn, and P.~Schweitzer, ``Online graph exploration: New
  results on old and new algorithms,'' \emph{Theoretical Computer Science},
  vol. 463, pp. 62--72, 2012.

\bibitem{petrich2011board}
J.~Petrich and K.~Subbarao, ``On-board wind speed estimation for uavs,'' in
  \emph{AIAA Guidance, Navigation, and Control Conference}, 2011, p. 6223.

\bibitem{algfoor2015comprehensive}
Z.~A. Algfoor, M.~S. Sunar, and H.~Kolivand, ``A comprehensive study on
  pathfinding techniques for robotics and video games,'' \emph{International
  Journal of Computer Games Technology}, vol. 2015, p.~7, 2015.

\bibitem{quigley2009ros}
M.~Quigley, K.~Conley, B.~Gerkey, J.~Faust, T.~Foote, J.~Leibs, R.~Wheeler, and
  A.~Y. Ng, ``Ros: an open-source robot operating system,'' in \emph{ICRA
  workshop on open source software}, vol.~3, no. 3.2.\hskip 1em plus 0.5em
  minus 0.4em\relax Kobe, Japan, 2009, p.~5.

\bibitem{wiki:xxx}
\BIBentryALTinterwordspacing
{Wikipedia contributors}, ``Universal transverse mercator coordinate system ---
  {Wikipedia}{,} the free encyclopedia,'' 2019, [Online; accessed
  29-April-2019]. [Online]. Available:
  \url{https://en.wikipedia.org/w/index.php?title=Universal_Transverse_Mercator_coordinate_system&oldid=891899172}
\BIBentrySTDinterwordspacing

\end{thebibliography}




\end{document}